\documentclass{article}

\usepackage{amsmath,amsfonts,bm}

\def\eqref#1{equation~\ref{#1}}

\def\ceil#1{\lceil #1 \rceil}
\def\floor#1{\lfloor #1 \rfloor}
\def\1{\bm{1}}

\DeclareMathAlphabet{\mathsfit}{\encodingdefault}{\sfdefault}{m}{sl}
\SetMathAlphabet{\mathsfit}{bold}{\encodingdefault}{\sfdefault}{bx}{n}

\DeclareMathOperator*{\argmin}{arg\,min}

\usepackage{microtype}
\usepackage{graphicx}
\usepackage{subfigure}
\usepackage{booktabs} 

\usepackage{hyperref}
\usepackage{url}

\usepackage{amsmath}
\usepackage{amsfonts}
\usepackage{amsthm}
\usepackage{amssymb}
\usepackage{bm}
\usepackage{mathtools}
\usepackage{graphicx}
\usepackage{multirow}
\usepackage{array}

\usepackage{xr}
\makeatletter
\newcommand*{\addFileDependency}[1]{
  \typeout{(#1)}
  \@addtofilelist{#1}
  \IfFileExists{#1}{}{\typeout{No file #1.}}
}
\makeatother
 \newcommand*{\myexternaldocument}[1]{%
    \externaldocument{#1}%
    \addFileDependency{#1.tex}%
    \addFileDependency{#1.aux}%
}
\myexternaldocument{appendix_icml2019}

\usepackage[accepted]{icml2019}

\icmltitlerunning{Approximation and Non-parametric Estimation of ResNet-type Convolutional Neural Networks}

\newcommand{\bbN}{{\mathbb N}}
\newcommand{\bbR}{{\mathbb R}}

\newcommand{\bbE}{{\mathbb E}}
\newcommand{\bbZ}{{\mathbb Z}}
\newcommand{\bmone}{{\bm 1}}
\newcommand{\bmzero}{{\bm 0}}

\newcommand{\vect}[1]{\mathrm{vec}(#1)}
\newcommand{\relu}{\mathrm{ReLU}}
\newcommand{\id}{\mathrm{id}}
\newcommand{\op}{\mathrm{op}}
\newcommand{\bmw}{{\bm w}}
\newcommand{\bmb}{{\bm b}}
\newcommand{\bmC}{{\bm C}}
\newcommand{\bmK}{{\bm K}}
\newcommand{\bmD}{{\bm D}}
\newcommand{\bmW}{{\bm W}}
\newcommand{\bmtheta}{{\bm \theta}}

\newcommand\dotdot[2]{#1_{#2, :, :}}

\newtheorem{prop}{Proposition}
\newtheorem{definition}{Definition}
\newtheorem{lemma}{Lemma}
\newtheorem{theorem}{Theorem}
\newtheorem{remark}{Remark}
\newtheorem{cor}{Corollary}

\newcounter{mainthms}

\newcolumntype{P}[1]{>{\centering\arraybackslash}p{#1}}

\begin{document}

\twocolumn[
\icmltitle{Approximation and Non-parametric Estimation of \\ ResNet-type Convolutional Neural Networks}



\icmlsetsymbol{equal}{*}
\icmlsetsymbol{work}{*}

\begin{icmlauthorlist}
\icmlauthor{Kenta Oono}{tokyo,pfn,work}
\icmlauthor{Taiji Suzuki}{tokyo,aip}
\end{icmlauthorlist}

\icmlaffiliation{tokyo}{Graduate School of Information Science and Technology, The University of Tokyo, Tokyo, Japan}
\icmlaffiliation{pfn}{Preferred Networks, Inc. (PFN), Tokyo, Japan}
\icmlaffiliation{aip}{Center for Advanced Intelligence Project (AIP), RIKEN, Tokyo, Japan}

\icmlcorrespondingauthor{Kenta Oono}{kenta\_oono@mist.i.u-tokyo.ac.jp}

\icmlkeywords{Convolutional Neural Network (CNN), Residual Neural Network (ResNet), learning theory, approximation theory, non-parametric estimation, block-sparse}

\vskip 0.3in
]



\printAffiliationsAndNotice{\textsuperscript{*}Work at the University of Tokyo.}  

\begin{abstract}
Convolutional neural networks (CNNs) have been shown to achieve optimal approximation and estimation error rates (in minimax sense) in several function classes. However, previous analyzed optimal CNNs are unrealistically wide and difficult to obtain via optimization due to sparse constraints in important function classes, including the H\"older class. We show a ResNet-type CNN can attain the minimax optimal error rates in these classes in more plausible situations -- it can be dense, and its width, channel size, and filter size are constant with respect to sample size. The key idea is that we can replicate the learning ability of Fully-connected neural networks (FNNs) by tailored CNNs, as long as the FNNs have \textit{block-sparse} structures. Our theory is general in a sense that we can automatically translate any approximation rate achieved by block-sparse FNNs into that by CNNs. As an application, we derive approximation and estimation error rates of the aformentioned type of CNNs for the Barron and H\"older classes with the same strategy.
\end{abstract}

\section{Introduction}

Convolutional neural network (CNN) is one of the most popular architectures in deep learning research, with various applications such as computer vision~\cite{NIPS2012_4824}, natural language processing~\cite{wu2016google}, and sequence analysis in bioinformatics~\cite{alipanahi2015predicting,zhou2015predicting}.
Despite practical popularity, theoretical justification for the power of CNNs is still scarce from the viewpoint of statistical learning theory.

For fully-connected neural networks (FNNs), there is a lot of existing work, dating back to the 80's, for theoretical explanation regarding their \textit{approximation} ability~\cite{cybenko1989approximation,barron1993universal,lu2017expressive,yarotsky2017error,lee17a,petersen2018optimal} and \textit{generalization} power~\cite{barron1994approximation,arora2018stronger,suzuki18a}.
See also surveys of earlier work by Pinkus~\yrcite{Pinkus2005} and Kainen et al.~\yrcite{kainen2013}.
Although less common compared to FNNs,
recent statistical learning theories for CNNs have been studied both about approximation ability \cite{zhou2018universality,yarotsky2018universal,petersen2018equivalence} and generalization power~\cite{zhou18a}.
Among others, Petersen and Voigtlaender~\yrcite{petersen2018equivalence} showed that any function realizable by an FNN is representable with an (equivariant) CNN with the same order of parameters.
This fact means virtually any approximation and estimation error rates achieved by FNNs can be achieved by CNNs, too.
In particular, because FNNs are optimal in minimax sense \cite{Tsybakov:2008:INE:1522486,gine2015mathematical} for several important function classes such as the H\"older class \cite{yarotsky2017error,schmidt2017nonparametric}, CNNs are also minimax optimal for these classes.

However, the optimal CNN obtained by the result of \cite{petersen2018optimal} can be unrealistically \textit{wide}: for $D$ variate $\beta$-H\"older case (see Definition \ref{def:holder-class}), its depth is $O(\log N)$, while its channel size is as large as $O(N^{\frac{D}{2\beta+D}})$ where $N$ is sample size.
To the best of our knowledge, no CNNs that achieve the minimax optimal rate in important function classes, including the H\"older class, can keep the number of units per layer constant with respect to $N$.
Thanks to recent techniques such as identity mappings~\cite{he2016deep,huang18b}, sophisticated initialization schemes~\cite{he2015delving,chen18i}, and normalization methods~\cite{ioffe15,miyato2018spectral}, architectures that are considerably deep and moderate channel size and width have become feasible.
Therefore, we would argue that there are growing demands for theories that can accommodate such constant-size architectures.

The other issue is impractical \textit{sparsity} constraints imposed on neural networks.
Existing literature \cite{schmidt2017nonparametric,suzuki2018adaptivity,imaizumi19a} proved the minimax optimal property of FNNs for several function classes.
However, they picked an estimator from a set of functions realizable by FNNs with a given number of non-zero parameters.
For example, Schmidt-Hieber~\yrcite{schmidt2017nonparametric} constructed an optimal FNN that has depth $O(\log N)$, width $O(N^{\alpha})$, and $O(N^\alpha \log N)$ non-zero parameters when the true function is $D$ variate $\beta$-H\"older. Here, $N$ is the sample size and $\alpha=\frac{D}{2\beta + D}$. It means the ratio of non-zero parameters (i.e., the number of non-zero parameters divided by the number of all parameters) is $\tilde{O}(N^{-\alpha})$.
To obtain such neural networks, we need to consider impractical combinatorial problems such as $L_0$ norm optimization.
Although we can obtain minimax optimal CNNs using the equivalence of CNNs and FNNs explained before, these CNNs also have the same order of sparsity.

In this paper, we show that CNNs can achieve minimax optimal approximation and estimation error rates, even with more plausible architectures. 
Specifically, we analyze the learning ability of ResNet-type~\cite{he2016deep} CNNs with ReLU activation functions~\cite{NIPS2012_4824}, which can be dense and have constant width, channel size, and filter size against the sample size.
There are mainly two reasons that motivate us to study this type of CNNs.
First, although ResNet is a de facto architecture in various practical applications, the minimax optimal property for ResNet has not been explored extensively.
Second, constant-width CNNs are critical building blocks not only in ResNet but also in various modern CNNs such as Inception ~\cite{szegedy2015going}, DenseNet~\cite{huang2017densely}, and U-Net~\cite{ronneberger2015u}, to name a few.

Our strategy is to emulate FNNs by constructing tailored ResNet-type CNNs similar to Zhou ~\yrcite{zhou2018universality} and Petersen and Voigtlaender~\yrcite{petersen2018equivalence}. The unique point of our method is to pay attention to a \textit{block-sparse} structure of an FNN, which roughly means a linear combination of multiple possibly dense FNNs. Block-sparseness decreases the model complexity from the combinatorial sparsity patterns and promotes better bounds.
Therefore, approximation and learning theories of FNNs often utilized it both implicitly or explicitly ~\cite{yarotsky2018universal,bolcskei2019optimal}. We first prove that if an FNN is block-sparse with $M$ blocks, we can realize the FNN with a ResNet-type CNN with $O(M)$ additional parameters.
In particular, if blocks in the FNN are dense, which is often true in typical settings, the increase of parameters in number is negligible. Therefore, the order of approximation rate of CNNs is the same as that of FNNs, and hence we can also show that the CNNs can achieve the same estimation error rate as the FNNs. We also note that CNN does not have sparse structures in general in this case. Although our primary interest is the H\"older class, this result is general in the sense that it is not restricted to a specific function class as long as we can approximate it using block-sparse FNNs.

To demonstrate the broad applicability of our methods, we derive approximation and estimation errors for two types of function classes with the same strategy: the Barron class (of parameter $s=2$, see Definition \ref{def:barron-class}) and H\"older class.
We prove, as corollaries, that our CNNs can achieve the approximation error of order $\tilde{O}(M^{-\frac{D+2}{2D}})$ for the Barron class and $\tilde{O}(M^{-\frac{\beta}{D}})$ for the $\beta$-H\"older class and the estimation error of order $\tilde{O}_P (N^{-\frac{D+2}{2(D+1)}})$ for the Barron class and $\tilde{O}_P (N^{-\frac{2\beta}{2\beta + D}})$ for the $\beta$-H\"older class, where $M$ is the number of parameters (we used $M$, which is same as the number of blocks, to indicate the parameter count because it will turn out that CNNs have $\Omega(M)$ blocks for these cases), $N$ is the sample size, and $D$ is the input dimension.
These rates are same as the ones for FNNs ever known in existing literature.
An important consequence of our theory is that the ResNet-type CNN can achieve the minimax optimal estimation error (up to logarithmic factors) for the H\"older class even if it can be dense, and its width, filter size, and channel size are constant against sample size. This fact is in contrast to existing work, where optimal FNNs or CNNs are inevitably sparse and have width or channel size going to infinity as $N\to \infty$.
Further, we prove minimax optimal CNNs can have constant-depth residual blocks for the H\"older case if we introduce signal scaling mechanisms to CNNs (see Definition \ref{def:masked-cnn-def}).

In summary, the contributions of our work are as follows:

\begin{itemize}
\item We develop general approximation theories for CNNs via ResNet-type architectures. If we can approximate a function with a block-sparse FNN with $M$ dense blocks, we can also approximate the function with a ResNet-type CNN at the same rate (Theorem \ref{thm:fnn-to-cnn}). The CNN is dense in general and is not assumed to have unrealistic sparse structures.
\item We derive the upper bound of the estimation error of ResNet-type CNNs (Theorem \ref{thm:estimation}). It gives a sufficient condition to obtain the same estimation error rate as FNNs (Corollary \ref{cor:estimation-error-wrt-sample-size}).
\item We apply our theory to the Barron and H\"older classes and derive the approximation (Corollary \ref{cor:approximation-barron} and \ref{cor:approximation-holder}) and estimation (Corollary \ref{cor:estimation-barron} and \ref{cor:estimation-holder}) error rates, which are identical to those for FNNs, even if the CNNs are dense and have constant width, channel size, and filter size with respect to sample size. This rate is minimax optimal for the H\"older case.
\item For the H\"older case, the optimal CNNs can additionally have constant-depth residual blocks if we introduce a scaling mechanism to identity mappings (Theorem \ref{thm:constant-depth-resnet-approximation} and \ref{thm:constant-depth-resnet-estimation}).
\end{itemize}

\section{Related Work}

In Table \ref{tab:cnn-comparison}, we highlight differences in CNN architectures between our work and work done by Zhou ~\yrcite{zhou2018universality} and Petersen and Voigtlaender~\yrcite{petersen2018equivalence}, which established approximation theories of CNNs via FNNs.

\begin{table*}[ht]
    \centering
        \caption{Comparison of CNN architectures.
        Function type: The function type CNNs can approximate. ``(Block-sparse) FNNs" means any function (blocks-sparse) FNNs can realize.
        Channel size: the number of channels needed to approximate a $\beta$-H\"older function with accuracy $\varepsilon$ measured by the sup norm.
        Sparsity: the ratio of non-zero parameters of optimal FNNs when the true function is $\beta$-H\"older ($N$ is the sample size).}
    \vspace{0.5em}
    \begin{tabular}{ccP{0.2\linewidth}c} \\\hline
                                             & \multirow{2}{*}{Zhou \yrcite{zhou2018universality}} & Petersen \& & \multirow{2}{*}{Ours} \\
                                            &   & Voigtlaender \yrcite{petersen2018equivalence} \hfill &  \\\hline

         CNN type                            & Conventional                & Conventional              & ResNet \\
         \multirow{2}{*}{Function type}      & \multirow{2}{*}{Barron ($s=2$)}& \multirow{2}{*}{FNNs}                           &  \multirow{2}{*}{Block-sparse FNNs}    \\
                                             &                             &                          &     \\
         \multirow{2}{*}{Channel size}                        & \multirow{2}{*}{N.A.}       & \multirow{2}{*}{$\tilde{O}(\varepsilon^{-\frac{D}{\beta}})$}        & \multirow{2}{*}{$O(1)$}  \\
                              &                             &                                &   \\
         \multirow{2}{*}{Sparsity}                           & \multirow{2}{*}{N.A.}                         & \multirow{2}{*}{$\tilde{O}(N^{-\frac{D}{2\beta+D}})$}                             & \multirow{2}{*}{$O(1)$}\\
         &&&\\\hline
    \end{tabular}
    \label{tab:cnn-comparison}
\end{table*}
First and foremost, Zhou only considered a specific function class --- the Barron class --- as a target function class, although we can apply their method to any function class realizable by a 2-layered ReLU FNN (i.e., a ReLU FNN with a single hidden layer).
Regarding architectures, they considered CNNs with a single channel and whose width is ``linearly increasing" \cite{zhou2018universality} layer by layer. For regression or classification problems, it is rare to use such an architecture. Besides, since they did not give the bound for the norm of parameters in approximating CNNs, we cannot derive the estimation error from their result.

Petersen and Voigtlaender \yrcite{petersen2018equivalence} fully utilized a group invariance structure of underlying input spaces to construct CNNs.
Such a structure makes theoretical analysis easier, especially for investigating the equivariance properties of CNNs, because it enables us to incorporate mathematical tools such as group theory, Fourier analysis, and representation theory~\cite{s.2018spherical}. Although their results are quite general in that we can apply it to any function approximated by FNNs, their assumption on group structures excludes the padding convolution layer, a popular type of convolution operation. Secondly, if we simply combine their result with the approximation result of Yarotsky~\yrcite{yarotsky2017error}, the CNN which optimally approximates $\beta$-H\"older function by the accuracy $\varepsilon$ (with respect to the sup-norm) has $\tilde{O}(\varepsilon^{-\frac{D}{\beta}})$ channels, which grows as $\varepsilon \to 0$ ($D$ is the input dimension). Finally, the ratio of non-zero parameters of optimal CNNs is $\tilde{O}(N^{-\frac{D}{2\beta + D}})$. That means the optimal CNNs get incredibly sparse as the sample size $N$ increases. One of the reasons for the large channel size and sparse structure is that their construction was not aware of the sparse internal structure of approximating FNNs, which motivates us to consider special structures of FNNs, the block-sparse structure.

Unlike these two studies, we employ padding- and ResNet-type CNNs, which have multiple channels, fixed-sized filters, and constant width.
Like Petersen and Voigtlaender~\yrcite{petersen2018equivalence}, we can apply our result to any function, as long as FNNs to be approximated are block-sparse, including the Barron and H\"older cases. If we use our theorem for these classes, we can show that the optimal CNNs can achieve the same approximation and estimation rates as FNNs, while they are dense, and the number of channels is independent of the sample size.

Finite-width neural networks have been studied in earlier work \cite{lu2017expressive,perekrestenko2018universal,fan2018universal}.
However, they only derived approximation abilities.
For finite-width networks, it is far from trivial to derive optimal estimation error rates from approximation results: if a network approximates a true function more accurately while restricting its capacity per layer, the neural network inevitably gets deeper.
Then, the model complexity of networks typically explodes exponentially as their depth increases, which makes it difficult to derive optimal estimation bounds.
We overcome this problem by sophisticated evaluation of model complexity using parameter rescaling techniques (see Section \ref{sec:barron}).

Due to its practical success, theoretical analysis for ResNet has been explored recently~\cite{lin2018resnet,lu18d,nitanda18a,huang18b}.
From the viewpoint of statistical learning theory, Nitanda and Suzuki~\yrcite{nitanda18a} and Huang et al.~\yrcite{huang18b} investigated the generalization power of ResNet from the perspective of boosting interpretation.
However, they did not derive precise estimation error rates for concrete function classes.
To the best of our knowledge, our theory is the first work to provide the estimation error rate of CNN classes that can accommodate the ResNet-type ones.

We import the approximation theories for FNNs, especially ones for the Barron and H\"older classes.
Originally Barron~\yrcite{barron1993universal} considered the Barron class with a parameter $s=1$ and an activation function $\sigma$ satisfying $\sigma(z)\to 1$ as $z\to \infty$ and $\sigma(z)\to 0$ as $z\to -\infty$.
Using this result, \citet{lee17a} proved that the composition of $n$ Barron functions with $s=1$ can be approximated by an FNN with $n+1$ layers.
Klusowski and Barron~\yrcite{klusowski2018approximation} studied its approximation theory with $s=2$ and proved that 2-layered ReLU FNNs with $M$ hidden units can approximate functions of this class with the order of $\tilde{O}(M^{-\frac{D+2}{2D}})$.
Yarotsky~\yrcite{yarotsky2017error} proved FNNs with $S$ non-zero parameters can approximate $D$ variate $\beta$-H\"older continuous functions with the order of  $\tilde{O}(S^{-\frac{\beta}{D}})$.
Using this bound, Schmidt-Hieber~\yrcite{schmidt2017nonparametric} proved that the estimation error of the ERM estimator is $\tilde{O}(N^{-\frac{2\beta}{2\beta + D}})$, which is minimax optimal up to logarithmic factors (see, e.g.,~\cite{Tsybakov:2008:INE:1522486}).

\section{Problem Setting}

We denote the set of positive integers by $\bbN_+ := \{1, 2, \ldots\}$ and
the set of positive integers less than or equal to $M\in \bbN_+$ by $[M] := \{1, \ldots, M\}$.
We define $a \vee b := \max(a, b)$ and $a\wedge b := \min(a, b)$ for $a, b\in \bbR$.

\subsection{Empirical Risk Minimization}\label{sec:problem-setting}

We consider a regression task in this paper.
Let $X$ be a $[-1, 1]^D$-valued random variable with an unknown probability distribution $\mathcal{P}_X$ and $\xi$ be an independent random noise drawn from the Gaussian distribution with an unknown variance $\sigma^2$ ($\sigma >0$): $\xi \sim \mathcal{N}(0, \sigma^2)$.
Let $f^\circ$ be an unknown deterministic function $f^{\circ}:[-1, 1]^D\to \bbR$ (we will characterize $f^\circ$ rigorously later).
We define a random variable $Y$ by $Y := f^\circ(X) + \xi$.
We denote the joint distribution of $(X, Y)$ by $\mathcal{P}$.
Suppose we are given a dataset $\mathcal{D} = ((x_1, y_1), \ldots, (x_N, y_N))$ independently and identically sampled from the distribution $\mathcal{P}$,
we want to estimate the true function $f^\circ$ from $\mathcal{D}$.

We evaluate the performance of an estimator by the squared error.
For a measurable function $f:[-1, 1]^D\to \bbR$, we define the \textit{empirical error} of $f$ by $\hat{\mathcal{R}}_\mathcal{D}(f):=\frac{1}{N} \sum_{n=1}^N (y_n - f(x_n))^2$ and the \textit{estimation error} by  $\mathcal{R}(f):=\mathbb{E}_{X, Y} \left[(f(X)-Y)^2 \right]$.
Given a subset $\mathcal{F}$ of measurable functions from $[-1, 1]^D$ to $\bbR$, we consider the \textit{clipped empirical risk minimization (ERM) estimator} $\hat{f}$ of $\mathcal{F}$ that satisfies
\begin{align*}
     \hat{f} := \mathrm{clip}[f_{\min}] \quad \text{where $f_{\min} \in \argmin_{f\in \mathcal{F}} \hat{\mathcal{R}}_\mathcal{D}(\mathrm{clip}[f])$}.
\end{align*}
Here, $\mathrm{clip}$ is a clipping operator defined by $\mathrm{clip}[f]:= (f\vee -\|f^{\circ}\|_\infty) \wedge \|f^{\circ}\|_\infty$. 
For a measurable function $f:[-1, 1]^D\to \bbR$, we define the $L_2$-norm (weighted by $\mathcal{P}_X$) and the sup norm of $f$ by  $\|f\|_{\mathcal{L}^2(\mathcal{P}_X)}:= \left( \int_{[-1, 1]^D} f^2(x) \mathrm{d} \mathcal{P}_X(x) \right)^\frac{1}{2}$ and $\|f\|_\infty := \sup_{x\in [-1, 1]^D} |f(x)|$, respectively.
Let $\mathcal{L}^2(\mathcal{P}_X)$ be the set of measurable functions $f$ such that $\|f\|_{\mathcal{L}^2(\mathcal{P}_X)} < \infty$ with the norm $\|\cdot\|_{\mathcal{L}^2(\mathcal{P}_X)}$.
The task is to estimate the \textit{approximation} error $\inf_{f\in \mathcal{F}} \|f-f^{\circ}\|_\infty$ and the \textit{estimation} error of the clipped ERM estimator: $\mathcal{R}(\hat{f}) - \mathcal{R}(f^\circ)$. Note that the estimation error is a random variable with respect to the choice of the training dataset $\mathcal{D}$.
By the definition of $\mathcal{R}$ and the independence of $X$ and $\xi$, the estimation error equals to $\|\hat{f} - f^\circ\|^2_{\mathcal{L}^2(P_X)}$.

\subsection{Convolutional Neural Networks}

In this section, we define CNNs used in this paper.
Let $K, C, C' \in \bbN_{+}$ be a filter size, input channel size, and output channel size, respectively.
For a filter $w = (w_{n, j, i})_{n\in [K], j\in [C'], i\in [C]}\in \bbR^{K\times C'\times C}$, we define the \textit{one-sided padding and stride-one convolution}\footnote{we discuss the difference of one-sided padding and two-sided padding in Appendix \ref{sec:padding-style}.} by $w$ as an order-4 tensor $L_D^w = ((L_D^w)^{\beta, j}_{\alpha, i}) \in \bbR^{D\times D\times C'\times C}$ defined by
\begin{align*}
(L_D^w)^{\beta, j}_{\alpha, i} :=
\begin{cases}
  w_{(\alpha - \beta + 1), j, i} & \text{if $0\leq \alpha - \beta \leq K-1$}, \\
  0  & \text{otherwise}.
\end{cases}
\end{align*}
Here, $i$ (resp. $j$) runs through $1$ to $C$ (resp. $C'$) and $\alpha$ and $\beta$ through $1$ to $D$.
Since we fix the input dimension $D$ throughout the paper, we omit the subscript $D$ and write it as $L^w$ if it is obvious from the context.
We can interpret $L^w$ as a linear mapping from $\bbR^{D\times C}$ to $\bbR^{D\times C'}$.
Specifically, for $x=(x_{\alpha, i})_{\alpha, i}\in \bbR^{D\times C}$, we define $(y_{\beta, j})_{\beta, j} = L^w(x) \in \bbR^{D\times C'}$ by
$$
y_{\beta, j} := \sum_{i, \alpha} (L^w)^{\beta, j}_{\alpha, i}\ x_{\alpha, i}.
$$

Next, we define the building blocks of CNNs: convolutional and fully-connected layers.
Let $K, C, C'\in \bbN_{+}$. For a weight tensor $w\in \bbR^{K\times C'\times C}$, a bias vector $b\in \bbR^{C'}$, and an activation function $\sigma: \bbR \to \bbR$, we define the \textit{convolutional layer} $\mathrm{Conv}^{\sigma}_{w, b}: \bbR^{D\times C}\to \bbR^{D\times C'}$ by $\mathrm{Conv}^{\sigma}_{w, b}(x):=\sigma(L^w(x) - \bmone_D \otimes b)$, where $\bmone_D$ is a $D$ dimensional vector consisting of 1's, $\otimes$ is the outer product of vectors, and $\sigma$ is applied in element-wise manner.
Similarly, let $W\in \bbR^{C'\times DC}$, $b\in \bbR^{C'}$, and $\sigma:\bbR\to \bbR$, we define the \textit{fully-connected layer} $\mathrm{FC}^{\sigma}_{W, b}: \bbR^{D\times C}\to \bbR^{C'}$ by $\mathrm{FC}^{\sigma}_{W, b}(a)=\sigma(W\vect{a} - b)$.
Here, $\vect{\cdot}$ is the vectorization operator that flattens a matrix into a vector.

Finally, we define the ResNet-type CNN as a sequential concatenation of one convolution block, $M$ residual blocks, and one fully-connected layer.
Figure \ref{fig:cnn-def} is the schematic view of the CNN we adopt in this paper.
\begin{figure}[t]
    \begin{center}
        \includegraphics[width=.9\linewidth]{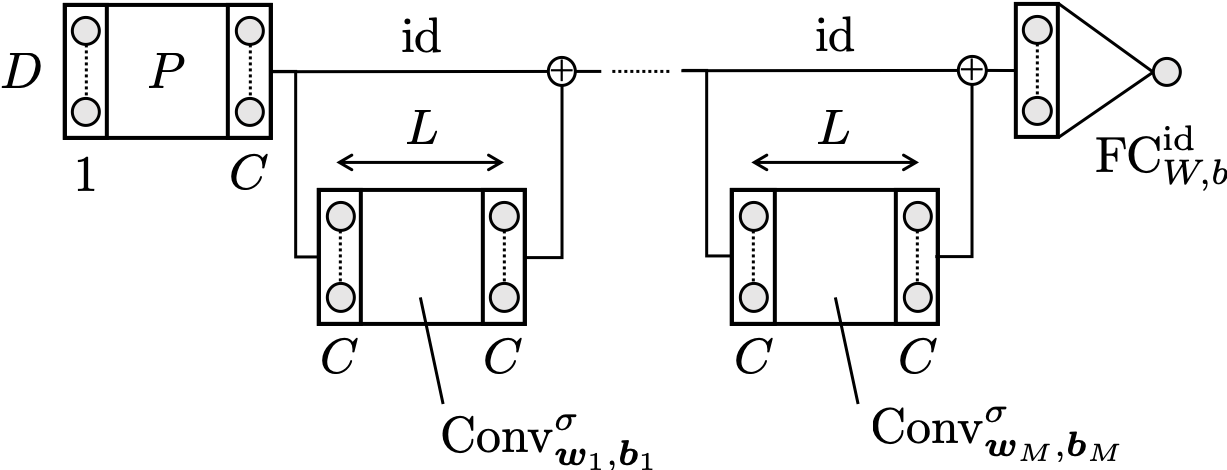}
    \end{center}
    \caption{ResNet-type CNN defined in Definition \ref{def:cnn-def}. Variables are as in Definition \ref{def:cnn-def}.}\label{fig:cnn-def}
\end{figure}

\begin{definition}[Convolutional Neural Networks (CNNs)]\label{def:cnn-def}
Let $M, L, C, K\in \bbN_{+}$, which will be the number of residual blocks and depth, channel size, and filter size of blocks, respectively.
For $m\in [M]$ and $l\in [L]$, let $w_m^{(l)}\in \bbR^{K\times C\times C}$ and $b_m^{(l)}\in \bbR^{C}$ be a weight tensor and bias of the $l$-th layer of the $m$-th block in the convolution part, respectively.
Finally, let $W\in \bbR^{DC \times 1}$ and $b\in \bbR$ be a weight matrix and a bias for the fully-connected layer part, respectively.
For $\bmtheta:=((w_m^{(l)})_{m,l}, (b_m^{(l)})_{m, l}, W, b)$ and an activation function $\sigma: \bbR \to \bbR$, we define  $\mathrm{CNN}^{\sigma}_{\bmtheta}:\bbR^{D}\to \bbR^{D}$, the CNN constructed from $\bmtheta$, by
\begin{align*}
    \mathrm{CNN}^{\sigma}_\bmtheta := \mathrm{FC}^{\id}_{W, b} & \circ 
    (\mathrm{Conv}^{\sigma}_{\bmw_M, \bmb_M}+ \id) \circ 
    \cdots \\&
    \circ (\mathrm{Conv}^{\sigma}_{\bmw_1, \bmb_1}+ \id) \circ
    P,
\end{align*}
where
$\mathrm{Conv}^{\sigma}_{\bmw_m, \bmb_m}:=
\mathrm{Conv}^{\sigma}_{w_m^{(L)}, b_m^{(L)}} \circ
\cdots \circ
\mathrm{Conv}^{\sigma}_{w_m^{(1)}, b_m^{(1)}}$,
$\id:\bbR^{D\times C} \to \bbR^{D\times C}$ is the identity function, and
$P: \bbR^{D}\to \bbR^{D\times C};
x \mapsto \begin{bmatrix}x & 0 & \cdots & 0 \end{bmatrix}$ is a padding operation that adds zeros to align the number of channels\footnote{Although $\mathrm{CNN}_\bmtheta^{\sigma}$ in this definition has a fully-connected layer, we refer to a stack of convolutional layers both with or without the final fully-connect layer as a CNN in this paper.}.
\end{definition}

We say a \textit{linear} convolutional layer or a \textit{linear} CNN when the activation function $\sigma$ is the identity function and a \textit{ReLU} convolution layer or a \textit{ReLU} CNN when $\sigma$ is ReLU, which is defined by $\relu(x) := x\vee 0$.
We borrow the term from ResNet and call $\mathrm{Conv}^{\sigma}_{\bmw_m, \bmb_m}$ ($m>0$) and $\id$ in the above definition the $m$-th \textit{residual block} and \text{identity mapping}, respectively.
We say $\bmtheta$ is \textit{compatible} with $(C, K)$ when each component of $\bmtheta$ satisfies the aforementioned dimension conditions.

For the number of blocks $M$, depth of residual blocks $L$, channel size $C$, filter size $K$, and norm parameters for convolution layers $B^{\mathrm{(conv)}}>0$ and for a fully-connected layer $B^{\mathrm{(fc)}} > 0$, we define $\mathcal{F}^{\mathrm{(CNN)}}_{M, L, C, K, B^{\mathrm{(conv)}}, B^{\mathrm{(fc)}}}$, the hypothesis class consisting of ReLU CNNs as
\begin{align*}
    \left\{\mathrm{CNN}_{\bmtheta}^{\relu}\ 
        \begin{tabular}{|l}
            $\mathrm{CNN}_{\bmtheta}^{\relu}$ has $M$ residual blocks, \\
            depth of each residual block is $L$, \\
            $\bmtheta$ is compatible with $(C, K)$, \\
            $\max_{m, l}\|w_m^{(l)}\|_{\infty} \vee \|b_m^{(l)}\|_\infty \leq B^{\mathrm{(conv)}}$,\\
            $\|W\|_{\infty} \vee \|b\|_\infty \leq B^{\mathrm{(fc)}}$
        \end{tabular}
    \right\}.
\end{align*}
Here, the domain of CNNs is restricted to $[-1, 1]^{D}$.
Note that we impose norm constraints to the convolution and fully-connected parts separately.
We emphasize that we do not impose any sparse constraints (e.g., restricting the number of non-zero parameters in a CNN to some fixed value) on CNNs, as opposed to previous literature~\cite{yarotsky2017error,schmidt2017nonparametric,imaizumi19a}.
We discuss differences between our CNN and the original ResNet~\cite{he2016deep} in Appendix \ref{sec:resnet-comparison}.

\subsection{Block-sparse Fully-connected Neural Networks}

In this section, we mathematically define FNNs we consider in this paper, in parallel with the CNN case.
Our FNN, which we coin a \textit{block-sparse} FNN, consists of $M$ possibly dense FNNs (blocks) concatenated in parallel, followed by a single fully-connected layer.
We sketch the architecture of a block-sparse FNN in Figure \ref{fig:block-sparse-fnn}.
\begin{figure}[t]
    \begin{center}
        \includegraphics[width=0.9\linewidth]{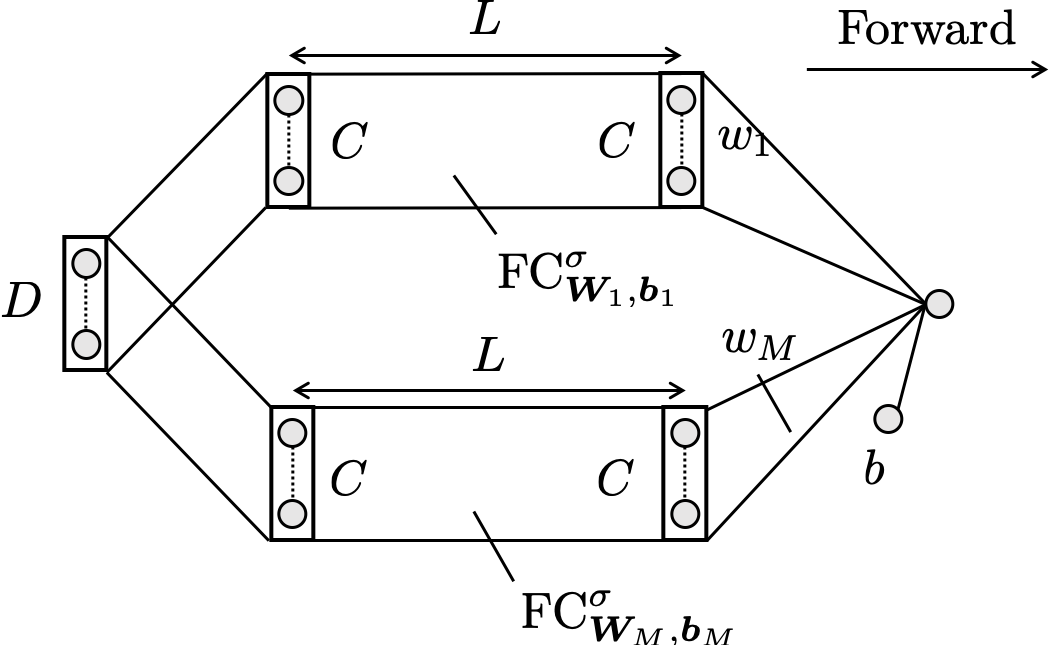}
    \end{center}
    \caption{Schematic view of a block-sparse FNN. Variables are as in Definition \ref{def:block-sparse-fnn}.}\label{fig:block-sparse-fnn}
\end{figure}

\begin{definition}[Fully-connected Neural Networks (FNNs)]\label{def:block-sparse-fnn}
Let $M, L, C\in \bbN_{+}$ be the number of blocks in an FNN, the depth and width of blocks, respectively.
Let $W_m^{(l)} \in \bbR^{C\times C}$ and $b_m^{(l)}\in \bbR^{C}$ be a weight matrix and a bias of the $l$-th layer of the $m$-th block
for $m\in [M]$ and $l\in [L]$, with the exception that $W_{m}^{(1)}\in \bbR^{C\times D}$.
Let $w_m\in \bbR^{C}$ be a weight (sub)vector of the final fully-connected layer corresponding to the $m$-th block and $b\in \bbR$ be a bias for the fully-connected layer.
For $\bmtheta = ((W_m^{(l)})_{m, l}, (b_m^{(l)})_{m, l}, (w_m)_m, b)$ and an activation function $\sigma: \bbR\to \bbR$, we define $\mathrm{FNN}_{\bmtheta}^{\sigma}: \bbR^D\to \bbR$, the block-sparse FNN constructed from $\bmtheta$, by
\begin{align*}
    \mathrm{FNN}_{\bmtheta}^{\sigma} := \sum_{m=1}^M w_m^{\top} \mathrm{FC}_{\bmW_m, \bmb_m}^\sigma(\cdot) - b,
\end{align*}
where $\mathrm{FC}_{\bmW_m, \bmb_m}^\sigma := \mathrm{FC}_{W_m^{(L)}, b_m^{(L)}}^\sigma \circ \cdots \circ \mathrm{FC}_{W_m^{(1)}, b_m^{(1)}}^\sigma$.
\end{definition}
We say $\bmtheta$ is \textit{compatible} with $C$ when each component of $\bmtheta$ matches the dimension conditions determined by the width parameter $C$, as we did in the CNN case.
When $L=1$, a block-sparse FNN is a 2-layered neural network with $C':= MC$ hidden units of the form $f(x) = \sum_{c=1}^{C'} b_c \sigma(a_c^{\top} x - t_c) - b$ where $a_c \in \bbR^D$ and $b_c, t_c, b\in \bbR$.

For the number of blocks $M$, depth $L$ and width $C$ of blocks, and norm parameters for the block part $B^{\mathrm{(bs)}}>0$ and for the final layer $B^{\mathrm{(fin)}} > 0$, we define $\mathcal{F}^{\mathrm{(FNN)}}_{M, L, C, B^{\mathrm{(bs)}}, B^{\mathrm{(fin)}}}$, the set of functions realizable by FNNs as
\begin{align*} 
    \left\{\mathrm{FNN}_{\bmtheta}^{\relu}\ 
            \begin{tabular}{|l}
                $\mathrm{FNN}_{\bmtheta}^{\relu}$ has $M$ blocks, \\
                depth of each block is $L$, \\
                $\bmtheta$ is compatible with $C$, \\
                $\max_{m, l} \|W^{(l)}_m\|_\infty \vee \|b^{(l)}_m\|_\infty \leq B^{\mathrm{(bs)}}$,\\
                $\max_{m} \|w_m\|_\infty \vee |b| \leq B^{\mathrm{(fin)}}$.
            \end{tabular}
        \right\},
\end{align*}
where the domain is again restricted to $[-1, 1]^D$.

\section{Main Theorems}

With the preparation in previous sections, we state the main results of this paper.
We only describe statements of theorems and corollaries in the main article.
All complete proofs are deferred to the supplemental material.

\subsection{Approximation}

Our first theorem claims that any block-sparse FNN with $M$ blocks is realizable by a ResNet-type CNN with fixed-sized channels and filters by adding $O(M)$ parameters.

\stepcounter{mainthms}
\begin{theorem}\label{thm:fnn-to-cnn}
Let $M, L, C\in \bbN_{+}$, $K\in \{2, \ldots D\}$ and $L_0 := \left\lceil \frac{D-1}{K-1}\right\rceil$.
Then, there exist $L' \leq L+L_0$, $C' \leq 4C$, and $K'\leq K$
such that, for any $B^{\mathrm{(bs)}}, B^{\mathrm{(fin)}} > 0$, any FNN in $\mathcal{F}^{\mathrm{(FNN)}}_{M, L, C, B^{\mathrm{(bs)}}, B^{\mathrm{(fin)}}}$ can be realized by a CNN in $\mathcal{F}^{\mathrm{(CNN)}}_{M, L', C', K', B^{\mathrm{(conv)}}, B^{\mathrm{(fc)}}}$.
Here, $B^{\mathrm{(conv)}} = \tilde{B}^{\mathrm{(bs)}}$ and $B^{\mathrm{(fc)}} = B^{\mathrm{(fin)}}(1\vee (\tilde{B}^{(\mathrm{bs})})^{-1})$, where $\tilde{B}^{\mathrm{(bs)}} = B^{\mathrm{(bs)}} \vee (B^{\mathrm{(bs)}})^{\frac{1}{L_{0}}}$.
\end{theorem}
In particular, if we can approximate a function with a block-sparse FNN with $O(M)$ parameters, we can also approximate the function with a ResNet-type CNN at the same rate.
By the definition of $\mathcal{F}^{\mathrm{(CNN)}}_{M, L', C', K', B^{\mathrm{(conv)}}}$, the CNN emulating the block-sparse FNN is dense and does not have sparse structures in general.

\subsection{Estimation}

Our second theorem bounds the estimation error of the clipped ERM estimator.
We denote 
$\mathcal{F}^{\mathrm{(FNN)}} = \mathcal{F}^{\mathrm{(FNN)}}_{M, L, C, B^{\mathrm{(bs)}}, B^{\mathrm{(fin)}}}$ and
$\mathcal{F}^{\mathrm{(CNN)}} = \mathcal{F}^{\mathrm{(CNN)}}_{M, L', C', K', B^{\mathrm{(conv)}}, B^{\mathrm{(fc)}}}$ in short.

\begin{theorem}\label{thm:estimation}
Let $f^\circ: \bbR^{D}\to \bbR$ be a measurable function and $B^{\mathrm{(bs)}}, B^{\mathrm{(fin)}} > 0$.
Let $M$, $L$, $C$, $K$, and $L_0$ as in Theorem \ref{thm:fnn-to-cnn}.
Suppose $L', C', K', B^{\mathrm{(conv)}}$ and $B^{\mathrm{(fc)}}$ satisfy 
$\mathcal{F}^{\mathrm{(FNN)}} \subset \mathcal{F}^{\mathrm{(CNN)}}$ (their existence is ensured by Theorem \ref{thm:fnn-to-cnn}).
Suppose that the covering nubmer of $\mathcal{F}^{\mathrm{(CNN)}}$ is larger than $2$.
Then, the clipped ERM estimator $\hat{f}$ of
$\mathcal{F}:= \{ \mathrm{clip}[f] \mid f \in \mathcal{F}^{\mathrm{(CNN)}} \}$ satisfies
\begin{align}\label{eq:estimation-bound}
    &\bbE_{\mathcal{D}} \|\hat{f} - f^\circ\|_{\mathcal{L}^2(\mathcal{P}_X)}^2 \nonumber \\
    &\leq C_0 \left(\inf_{f} \|f-f^\circ\|^2_{\infty} + \frac{\tilde{F}^2}{N} \Lambda_2 \log (2\Lambda_1BN) \right).
\end{align}
Here, $f$ ranges over $\mathcal{F}^\mathrm{(FNN)}$, $C_0>0$ is a universal constant, $\tilde{F} := \frac{\|f^\circ\|_\infty}{\sigma} \vee \frac{1}{2}$, and $B := B^{\mathrm{(conv)}}\vee B^{\mathrm{(fc)}}$. $\Lambda_1 = \Lambda_1(\mathcal{F}^{\mathrm{(CNN)}})$ and $\Lambda_2 = \Lambda_2(\mathcal{F}^{\mathrm{(CNN)}})$ are defined by
\begin{align*}
    \Lambda_1 &:= (2M+3)C'D (1\vee B^{\mathrm{(fc)}}) (1\vee B^\mathrm{(conv)}) \varrho \varrho^+, \\
    \Lambda_2 &:= ML'\left({C'}^2K'+C'\right) + C'D + 1,
\end{align*}
where
$\varrho := (1+\rho)^M$,
$\varrho^+ := 1 + ML'\rho^+$,
$\rho := (C'K' B^{\mathrm{(conv)}})^{L'}$,
and $\rho^+ := (1\vee C'K' B^{\mathrm{(conv)}})^{L'}$.
\end{theorem}
The first term of (\ref{eq:estimation-bound}) is the approximation error achieved by $\mathcal{F}^{\mathrm{(FNN)}}$.
On the other hand, the second term of (\ref{eq:estimation-bound}) represents the model complexity of $\mathcal{F}^{\mathrm{(CNN)}}$ since $\Lambda_1$ and $\Lambda_2$ are determined by the architectural parameters of $\mathcal{F}^{\mathrm{(CNN)}}$ --- $\Lambda_1$ corresponds to the Lipschitz constant of a function realized by a CNN and $\Lambda_2$ is the number of parameters, including zeros, of a CNN.
There is a trade-off between these two terms.
Using appropriately chosen $M$ to balance them, we can evaluate the order of estimation error with respect to the sample size $N$.

\begin{cor}\label{cor:estimation-error-wrt-sample-size}
Under the same assumptions as Theorem \ref{thm:estimation},
suppose further $\log \Lambda_1B = \tilde{O}(1)$ as a function of $M$.
If $\inf_{f\in \mathcal{F}^\mathrm{(FNN)}} \|f-f^\circ\|^2_{\infty} = \tilde{O}(M^{-\gamma_1})$ and
$\Lambda_2 = \tilde{O}(M^{\gamma_2})$ for some constants $\gamma_1, \gamma_2 >0$ independent of $M$,
 then, the clipped ERM estimator $\hat{f}$ of $\mathcal{F}$ achieves the estimation error $\|f^\circ - \hat{f}\|_{\mathcal{L}_2(\mathcal{P}_X)}^2 = \tilde{O}_{P}(N^{-\frac{2\gamma_1}{2\gamma_1 + \gamma_2}})$.
\end{cor}

\section{Application}

\subsection{Barron Class}\label{sec:barron}

The Barron class is an example of the function class that can be approximated by block-sparse FNNs.
We employ the definition of Barron functions used in~\cite{klusowski2018approximation}.
\begin{definition}[Barron class]\label{def:barron-class}
We call a measurable function $f^\circ:[-1, 1]^D\to \bbR$ a Barron function of a parameter $s > 0$ if $f^\circ$ admits the Fourier representation (i.e., $f^\circ(x) = \check{\mathcal{F}} \mathcal{F}[f^\circ]$) and $\int_{\bbR^D} \|w\|_2^s \left| \mathcal{F}[f^\circ](w) \right| \mathrm{d}w < \infty$.
Here, $\mathcal{F}$ and $\check{\mathcal{F}}$ are the Fourier and inverse Fourier transformations, respectively.
\end{definition}

Klusowski and Barron \yrcite{klusowski2018approximation} studied approximation of the Barron function $f^\circ$ with the parameter $s=2$ by a linear combination of $M$ ridge functions (i.e., a $2$-layered ReLU FNN). Specifically, they showed that there exists a function $f_M$ of the form
\begin{align}
    f_M := f^{\circ}(0) + \nabla f^{\circ\top} (0) x + \frac{1}{M}\sum_{m=1}^M b_m (a_m^\top x - t_m)_+ \label{eq:resaling-barron}
\end{align}
with $|b_m| \leq 1$, $\|a_m\|_1 = 1$, and $|t_m|\leq 1$, such that $\|f^{\circ} - f_M\|_\infty = \tilde{O}(M^{-\left(\frac{1}{2} + \frac{1}{D}\right)})$.
Using this approximator $f_M$, we can derive the same approximation order using CNNs by applying Theorem \ref{thm:fnn-to-cnn} with $L = 1$ and $C = 1$.

\begin{cor}\label{cor:approximation-barron}
Let $f^\circ:[-1, 1]^{D}\to \bbR$ be a Barron function with the parameter $s=2$ such that $f^{\circ}(0) = 0$ and $\nabla f^\circ(0) = \bmzero_D$.
Then, for any $K \in \{2, \ldots, D\}$, there exists a CNN $f^{(\mathrm{CNN})}$ with $M$ residual blocks, each of which has depth $O(1)$ and at most $4$ channels, and whose filter size is at most $K$, such that $\|f^\circ - f^{(\mathrm{CNN})}\|_\infty = \tilde{O}(M^{-\left(\frac{1}{2} + \frac{1}{D}\right)})$.
\end{cor}
Note that this rate is same as the one obtained for FNNs \cite{klusowski2018approximation}.

We have one design choice when we apply Corollary \ref{cor:estimation-error-wrt-sample-size} in order to derive the estimation error: how to set $B^{\mathrm{(bs)}}$ and $B^{\mathrm{(fin)}}$?
Looking at the definition of $f_M$, a naive choice would be $B^{\mathrm{(bs)}} := 1$ and $B^{\mathrm{(fin)}} := M^{-1}$.
However, this cannot satisfy the assumption on $\Lambda_1$ of Corollary \ref{cor:estimation-error-wrt-sample-size}, due to the term
$\varrho = (1+\rho)^{M}$. 
We want the logarithm of $\Lambda_1$ to be $\tilde{O}(1)$ as a function of $M$.
To do that, we change the \textit{relative scale} between parameters in the block-sparse part and the fully-connected part using the homogeneous property of the ReLU function: $\relu(ax) = a\relu(x)$ for $a > 0$.
The rescaling operation enables us to choose $B^{\mathrm{(bs)}} := M^{-1}$ and $B^{\mathrm{(fin)}} = 1$ to meet the assumption of Corollary \ref{cor:estimation-error-wrt-sample-size}.
By setting $\gamma_1= \frac{1}{2} + \frac{1}{D}$ and $\gamma_2 = 1$, we obtain the desired estimation error.

\begin{cor}\label{cor:estimation-barron}
Let $f^\circ:[-1, 1]^{D}\to \bbR$ be a Barron function with the parameter $s=2$ such that $f^{\circ}(0) = 0$ and $\nabla f^\circ(0) = \bmzero_D$.
Let $K\in \{2, \ldots, D\}$. There exist the number of residual blocks $M=O(N^{\frac{D}{2+2D}})$, depth of each residual block $L = O(1)$, channel size $C=O(1)$, and norm bounds $B^{\mathrm{(conv)}}, B^{\mathrm{(fc)}} > 0$ such that for sufficiently large $N$, the clipped ERM estimator $\hat{f}$ of $\{\mathrm{clip}[f] \mid  f \in \mathcal{F}^{\mathrm{(CNN)}}_{M, L, C, K, B^{\mathrm{(conv)}}, B^{\mathrm{(fc)}}} \}$ achieves the estimation error $\|f^\circ - \hat{f}\|_{\mathcal{L}_2(\mathcal{P}_X)}^2 = \tilde{O}_{P}(N^{-\frac{D+2}{2(D+1)}})$.
\end{cor}

\subsection{H\"older Class}

We next consider the approximation and error rates of CNNs when the true function $f^\circ$ is an H\"older function.

\begin{definition}[H\"older class]\label{def:holder-class}
Let $\beta > 0$. A function $f^\circ:[-1, 1]^D \to \bbR$ is called a $\beta$-H\"older function if
\begin{align*}
    \|f^{\circ}\|_\beta &:= \sum_{0\leq |\alpha| < \floor{\beta}} \left\|\partial^\alpha f^\circ \right\|_\infty\\
        & \quad + \sum_{|\alpha| = \floor{\beta}} \sup_{x\not = y} \frac{|\partial^\alpha f^\circ(x) - \partial^\alpha f^\circ(y)|}{|x-y|^{\beta - \floor{\beta}}} < \infty.
\end{align*}
Here, $\alpha = (\alpha_1, \ldots, \alpha_D)$ is a multi-index. That is,
$\partial^\alpha f := \frac{\partial^{|\alpha|} f }{\partial x^{\alpha_1}_1\cdots \partial x^{\alpha_D}_D}$ and $|\alpha| := \sum_{d=1}^D \alpha_d$.
\end{definition}

Yarotsky \yrcite{yarotsky2017error} showed that FNNs with $S$ non-zero parameters can approximate any $D$ variate $\beta$-H\"older function with the order of $\tilde{O}(S^{-\frac{\beta}{D}})$.
Schmidt-Hieber \yrcite{schmidt2017nonparametric} also proved a similar statement using a different construction method.
They only specified the width\footnote{Yarotsky~\yrcite{yarotsky2017error} didn't specify the width of FNNs.}, depth, and non-zero parameter counts of the approximating FNN and did not write in detail how non-zero parameters are distributed in the statements explicitly (see Theorem 1 of~\cite{yarotsky2017error} and Theorem 5 of~\cite{schmidt2017nonparametric}).
However, if we carefully look at their proofs, we can transform the FNNs they constructed into block-sparse ones (see Lemma \ref{lem:schmidt-hieber-thm5} of the supplemental material).
Therefore, we can apply Theorem \ref{thm:fnn-to-cnn} to these FNNs.
To meet the assumption of Corollary \ref{cor:estimation-error-wrt-sample-size}, we again rescale the parameters of the FNNs, as we did in the Barron-class case so that $\log \Lambda_1 = \tilde{O}(1)$.
We can derive the approximation and estimation errors by setting $\gamma_1=\frac{\beta}{D}$ and $\gamma_2=1$.

\begin{cor}\label{cor:approximation-holder}
    Let $\beta > 0$ and $f^\circ: [-1, 1]^D \to \bbR$ be a $\beta$-H\"older function.
    Then, for any $K \in \{2, \ldots, D\}$, there exists a CNN $f^{\mathrm{(CNN)}}$ with $O(M)$ residual blocks, each of which has depth $O(\log M)$ and $O(1)$ channels, and whose filter size is at most K, such that $\|f^\circ - f^{(\mathrm{CNN})}\|_\infty = \tilde{O}(M^{-\frac{\beta}{D}})$.
\end{cor}

\begin{cor}\label{cor:estimation-holder}
Let $\beta > 0$ and $f^\circ: [-1, 1]^D \to \bbR$ be a $\beta$-H\"older function.
For any $K\in \{2, \ldots, D\}$, there exist the number of residual blocks $M=O(N^{\frac{D}{2\beta + D}})$, depth of each residual block $L = O(\log N)$, channel size $C=O(1)$, and norm bounds $B^{\mathrm{(conv)}}, B^{\mathrm{(fc)}} > 0$ such that for sufficiently large $N$, the clipped ERM estimator $\hat{f}$ of $\{\mathrm{clip}[f] \mid  f \in \mathcal{F}^{\mathrm{(CNN)}}_{M, L, C, K, B^{\mathrm{(conv)}}, B^{\mathrm{(fc)}}} \}$ achieves the estimation error $\|f^\circ - \hat{f}\|_{\mathcal{L}_2(\mathcal{P}_X)}^2 = \tilde{O}_{P}(N^{-\frac{2\beta}{2\beta + D}})$.
\end{cor}

Since the estimation error rate of the $\beta$-H\"older class is $O_P(N^{-\frac{2\beta}{2\beta + D}})$ (see, e.g.,~\cite{Tsybakov:2008:INE:1522486}),
Corollary \ref{cor:estimation-holder} implies that our CNN can achieve the minimax optimal rate up to logarithmic factors even though it can be dense and its width $D$, channel size $C$, and filter size $K$ are constant with respect to the sample size $N$.

\section{Optimal CNNs with Constant-depth Blocks}

In the previous section, we proved the optimality of dense and narrow ResNet-type CNNs for the H\"older class. However, the constructed CNN can have residual blocks whose depth is as large as $O(\log N)$. Such an architecture differs from practically successful ResNets because they usually have relatively shallow (e.g., 2- or 3-layered) networks as residual blocks. We hypothesize that the essence of the problem resides in the difference of scales between identity connections and residual blocks. Therefore, we consider another type of CNNs that admits scaling schemes of intermediate signals in order to overcome this problem. Among others, we consider the simplest scaling method, which zeros out some channels in identity mappings.
\begin{definition}[Masked CNNs]\label{def:masked-cnn-def}
Let $M, L, C, K\in \bbN_{+}$.
Let $w_m^{(l)}\in \bbR^{K\times C\times C}$, $b_m^{(l)}\in \bbR^{C}$, $W\in \bbR^{DC \times 1}$ and $b\in \bbR$ be parameters of CNNs for $m\in [M]$ and $l\in [L]$.
Let $z_m = (z_{m, 1}, \ldots, z_{m, C}) \in \{0, 1\}^{C}$ be a mask for the $m$-th identity mapping.
For $\bmtheta:=((w_m^{(l)})_{m,l}, (b_m^{(l)})_{m, l}, W, b, (z_m)_m)$ and an activation function $\sigma: \bbR \to \bbR$, we define  $\mathrm{mCNN}^{\sigma}_{\bmtheta}:\bbR^{D}\to \bbR^{D}$, the masked CNN constructed from $\bmtheta$, by
\begin{align*}
    \mathrm{mCNN}^{\sigma}_\bmtheta := \mathrm{FC}^{\id}_{W, b} & \circ 
    (\mathrm{Conv}^{\sigma}_{\bmw_M, \bmb_M}+ J_M)  \circ 
    \cdots \\&
    \circ (\mathrm{Conv}^{\sigma}_{\bmw_1, \bmb_1}+ J_1) \circ
    P,
\end{align*}
where
$J_{m}: \bbR^{D\times C}\to \bbR^{D\times C}$ is a channel wise mask operation defined by  $[x_1\ \cdots\ x_C] \mapsto [z_{m, 1} x_1 \ \cdots \ z_{m, C} x_C]$.
\end{definition}
By definition, plain ResNet-type CNNs in Definition \ref{def:cnn-def} are a special case of masked CNNs.
Note that we do not restrict the number of non-zero mask elements.
Therefore, although masks take discrete values, we can obtain approximated ERM estimators via sparse optimization techniques.
We say $\bmtheta$ is compatible with $(C, K)$ when $\bmtheta$ satisfies the dimension conditions as we did in Definition \ref{def:cnn-def}.
We define $\mathcal{G}_{M, L, C, K, B^{\mathrm{(conv)}}, B^{\mathrm{(fc)}}}$ by
\begin{align*}
    \left\{\mathrm{mCNN}_{\bmtheta}^{\relu}\ 
        \begin{tabular}{|l}
            $\mathrm{mCNN}_{\bmtheta}^{\relu}$ has $M$ residual blocks, \\
            depth of each residual block is $L$, \\
            $\bmtheta$ is compatible with $(C, K)$, \\
            $\max_{m, l}\|w_m^{(l)}\|_{\infty} \vee \|b_m^{(l)}\|_\infty \leq B^{\mathrm{(conv)}}$,\\
            $\|W\|_{\infty} \vee \|b\|_\infty \leq B^{\mathrm{(fc)}}$
        \end{tabular}
    \right\}.
\end{align*}
The above definition treats the mask pattern $z = (z_m)_m$ as learnable parameters. We can also treat $z$ as fixed during training and search for the best $z$ as an architecture search.
The following theorems show that masked CNNs can approximate and estimate any H\"older function optimally even if the depth of residual blocks is specified \textit{a priori}.
We treat $L$ as a constant against $M$ in the theorems.
\begin{theorem}\label{thm:constant-depth-resnet-approximation}
    Let $f^\circ: [-1, 1]^D \to \bbR$ be a $\beta$-H\"older function.
    For any $K \in \{2, \ldots, D\}$ and $L \in \bbN_{+}$, there exists a CNN $f^{\mathrm{(CNN)}}$ with $O(M\log M)$ residual blocks, each of which has depth $L$ and $O(1)$ channels, and whose filter size is at most $K$, such that $\|f^\circ - f^{(\mathrm{CNN})}\|_\infty = \tilde{O}(M^{-\frac{\beta}{D}})$.
\end{theorem}
\begin{theorem}\label{thm:constant-depth-resnet-estimation}
Let $f^\circ: [-1, 1]^D \to \bbR$ be a $\beta$-H\"older function.
For any $K\in \{2, \ldots, D\}$ and $L\in \bbN_+$,
there exist the number of residual blocks $\tilde{M}=O(N^{\frac{D}{2\beta + D}}\log N)$, channel size $C=O(1)$, and norm bounds $B^{\mathrm{(conv)}}, B^{\mathrm{(fc)}} > 0$ such that for sufficiently large $N$, the clipped ERM estimator $\hat{f}$ of $\{\mathrm{clip}[f] \mid  f \in \mathcal{G}_{\tilde{M}, L, C, K, B^{\mathrm{(conv)}}, B^{\mathrm{(fc)}}} \}$ achieves the estimation error $\|f^\circ - \hat{f}\|_{\mathcal{L}_2(\mathcal{P}_X)}^2 = \tilde{O}_{P}(N^{-\frac{2\beta}{2\beta + D}})$.
\end{theorem}

\section{Conclusion}

In this paper, we established new approximation and statistical learning theories for CNNs by utilizing the ResNet-type architecture of CNNs and the block-sparse structure of FNNs.
We proved that any block-sparse FNN with $M$ blocks is realizable by a CNN with $O(M)$ additional parameters.
Then, we derived the approximation and estimation error rates for CNNs from those for block-sparse FNNs.
Our theory is general in that it does not depend on a specific function class as long as we can approximate it with block-sparse FNNs.
Using this theory, we derived approximation and error rates for the Barron and H\"older classes in almost the same manner and showed that the estimation error of CNNs is the same as that of FNNs, even if CNNs are dense and have constant channel size, filter size, and width with respect to the sample size.
We can additionally make the depth of residual blocks constant if we allow identity mappings to have scaling schemes.
The key techniques were careful evaluations of the Lipschitz constant and non-trivial weight parameter rescaling of NNs.

One of the interesting open questions is the role of weight rescaling.
We critically use the homogeneous property of the ReLU to change the relative scale between the block-sparse and fully-connected parts. If it were not for this property, the estimation error rate would be worse.
The general theory for rescaling, not restricted to the Barron nor H\"older classes, would be beneficial for a deeper understanding of the relationship between the approximation and estimation capabilities of FNNs and CNNs.

Another question is when the approximation and estimation error rates of CNNs can \textit{exceed} that of FNNs.
We can derive the same rates as FNNs essentially because we can realize block-sparse FNNs using CNNs with the same order of parameters (see Theorem \ref{thm:fnn-to-cnn}). If we can find some special structures of FNNs -- like repetition, the CNNs might need fewer parameters and can achieve a better estimation error rate. 
Note that there is no hope for enhancement for the H\"older case since the estimation rate using FNNs is already minimax optimal (up to logarithmic factors).
It is left for future research which functions classes and constraints of FNNs, like block-sparseness, we should choose.

\section*{Acknowledgements}

We thank Kohei Hayashi for improving on the draft, Wei Lu for commenting on the preprint version of the paper and pointing out its errata, 
Yunfei Yang for pointing out the technical issues of Lemma~\ref{lem:realize-ridge-with-cnn}, and anonymous reviewers for fruitful discussion and positive feedback and comments.
TS was partially supported by MEXT Kakenhi (26280009, 15H05707,
18K19793 and 18H03201), Japan Digital Design and JSTCREST.


\bibliography{main_icml2019}
\bibliographystyle{icml2019}

\onecolumn
\clearpage
\normalsize
\section*{Appendix}
\appendix
In this supplemental material, we give the proofs of theorems and corollaries in the main article. We prove them in a more general form.
Specifically, we allow CNNs to have residual blocks with different depths and each residual block to have varying numbers of channels and filter sizes.
Similarly, FNNs can have blocks with different depths, and the width of a block can be non-constant.

\section{Notation}

For tensor $a$, we define the positive part of $a$ by $a_+:=a\vee 0$ where the maximum operation is performed element-wise. Similarly, the negative part of $a$ is defined as $a_-:=-a\vee 0$.
Note that $a = a_+ - a_-$ holds for any tensor $a$.
For normed spaces $(V, \|\cdot\|_V), (W, \|\cdot\|_W)$ and a linear operator $T:V\to W$ we denote the operator norm of $T$ by $\|T\|_\mathrm{op}:=\sup_{\|v\|_V = 1} \|Tv\|_W$.
For a sequence $\bmw = (w^{(1)}, \ldots, w^{(L)})$ and $l\leq l'$, we denote its subsequence from the $l$-th to $l'$-th elements by $\bmw[l:l'] := (w^{(l)}, \ldots, w^{(l')})$.

\section{Definitions}

We define general types of ResNet-type CNNs and block-sparse FNNs.

\begin{definition}[Convolutional Neural Networks (CNNs)]\label{def:cnn-def-extend}
Let $M\in \bbN_{+}$ and $L_m\in \bbN_{+}$, which will be the number of residual blocks and the depth of $m$-th block, respectively.
Let $C_m^{(l)}, K_{m}^{(l)}$ be the channel size and filter size of the $l$-th layer of the $m$-th block for $m\in[M]$ and $l\in [L_m]$.
We assume $C_1^{(L_1)}=\cdots=C_{M}^{(L_M)}$ and denote it by $C^{(0)}$.
Let $w_m^{(l)}\in \bbR^{K^{(l)}_m\times C^{(l)}_m \times C^{(l-1)}_m}$ and $b_m^{(l)}\in \bbR$ be the weight tensors and biases of $l$-th layer of the $m$-th block in the convolution part, respectively.
Here $C_m^{(0)}$ is defined as $C^{(0)}$.
Finally, let $W\in \bbR^{D \times C^{(L_0)}_0}$ and $b\in \bbR$ be the weight matrix and the bias for the fully-connected layer part, respectively.
For $\bmtheta:=((w_m^{(l)})_{m,l}, (b_m^{(l)})_{m, l}, W, b)$ and an activation function $\sigma: \bbR \to \bbR$, we define  $\mathrm{CNN}^{\sigma}_{\bmtheta}:\bbR^{D}\to \bbR^{D}$, the CNN constructed from $\bmtheta$, by
\begin{align*}
    \mathrm{CNN}^{\sigma}_\bmtheta := \mathrm{FC}^{\id}_{W, b}\circ 
    (\mathrm{Conv}^{\sigma}_{\bmw_M, \bmb_M}+ \id) \circ 
    \cdots
    \circ (\mathrm{Conv}^{\sigma}_{\bmw_1, \bmb_1}+ \id) \circ
    P,
\end{align*}
where
$\mathrm{Conv}^{\sigma}_{\bmw_m, \bmb_m}:=
\mathrm{Conv}^{\sigma}_{w_m^{(L_m)}, b_m^{(L_m)}} \circ
\cdots \circ
\mathrm{Conv}^{\sigma}_{w_m^{(1)}, b_m^{(1)}}$,
$\id:\bbR^{D\times C^{(0)}} \to \bbR^{D\times C^{(0)}}$ is the identity function, and
$P: \bbR^{D}\to \bbR^{D\times C^{(0)}};
x \mapsto \begin{bmatrix}x & 0 & \cdots & 0 \end{bmatrix}$ is a padding operation that adds zeros to align the number of channels.
\end{definition}

\begin{definition}[Fully-connected Neural Networks (FNNs)]\label{def:block-sparse-fnn-extend}
Let $M\in \bbN_{+}$ be the number of blocks in an FNN.
Let $\bmD_m = (D^{(1)}_m, \ldots, D^{(L_m)}_m) \in \bbN^{L_m}_{+}$ be the sequence of intermediate dimensions of the $m$-th block, where $L_m \in \bbN_{+}$ is the depth of the $m$-th block for $m\in [M]$.
Let $W_m^{(l)} \in \bbR^{D^{(l)}_m\times D^{(l-1)}_m}$ and $b_m^{(l)}\in \bbR^{D^{(l)}_m}$ be the weight matrix and the bias of the $l$-th layer of $m$-th block (with the convention $D^{(0)}_m = D$).
Let $w_m\in \bbR^{D^{(L_m)}_m}$ be the weight (sub)vector of the final fully-connected layer corresponding to the $m$-th block and $b\in \bbR$ be the bias for the last layer.
For $\bmtheta = ((W_m^{(l)})_{m, l}, (b_m^{(l)})_{m, l}, (w_m)_m, b)$ and an activation function $\sigma: \bbR\to \bbR$, we define $\mathrm{FNN}_{\bmtheta}^{\sigma}: \bbR^D\to \bbR$, the block-sparse FNN constructed from $\bmtheta$, by
\begin{align*}
    \mathrm{FNN}_{\bmtheta}^{\sigma} := \sum_{m=1}^M w_m^{\top} \mathrm{FC}_{\bmW_m, \bmb_m}^\sigma(\cdot) - b,
\end{align*}
where $\mathrm{FC}_{\bmW_m, \bmb_m}^\sigma := \mathrm{FC}_{W_m^{(L_m)}, b_m^{(L_m)}}^\sigma \circ \cdots \mathrm{FC}_{W_m^{(1)}, b_m^{(1)}}^\sigma$.
\end{definition}

Figure \ref{fig:cnn-def-extend} shows the schematic view of a ResNet-type CNNs defined in Definition \ref{def:cnn-def-extend} and Figure \ref{fig:block-sparse-fnn-extend} shows that of Definition \ref{def:block-sparse-fnn-extend}.
Definition \ref{def:cnn-def-extend} is reduced to Definition \ref{def:cnn-def} by setting $L_m = L$, $\bmC = (C)_{m, l}$ and $\bmK = (K)_{m, l}$.
Similarly, Definition \ref{def:block-sparse-fnn} is a special case of  Definition \ref{def:block-sparse-fnn-extend} where $L_m = L$ and $\bmD=(C)_{m, l}$.
Correspondingly, we denote the set of functions realizable by CNNs and FNNs by $\mathcal{F}^{\mathrm{(CNN)}}_{\bmC, \bmK, B^{\mathrm{(conv)}}, B^{\mathrm{(fc)}}}$ and $\mathcal{F}^{\mathrm{(FNN)}}_{\bmD, B^{\mathrm{(bs)}}, B^{\mathrm{(fin)}}}$, respectively \footnote{Note that information of $M$ and $L_m$ are included in $\bmC$, $\bmK$, and $\bmD$. Therefore, we do not have to put them as subscripts}.
\begin{figure}[t]
    \begin{center}
        \includegraphics[width=.8\linewidth]{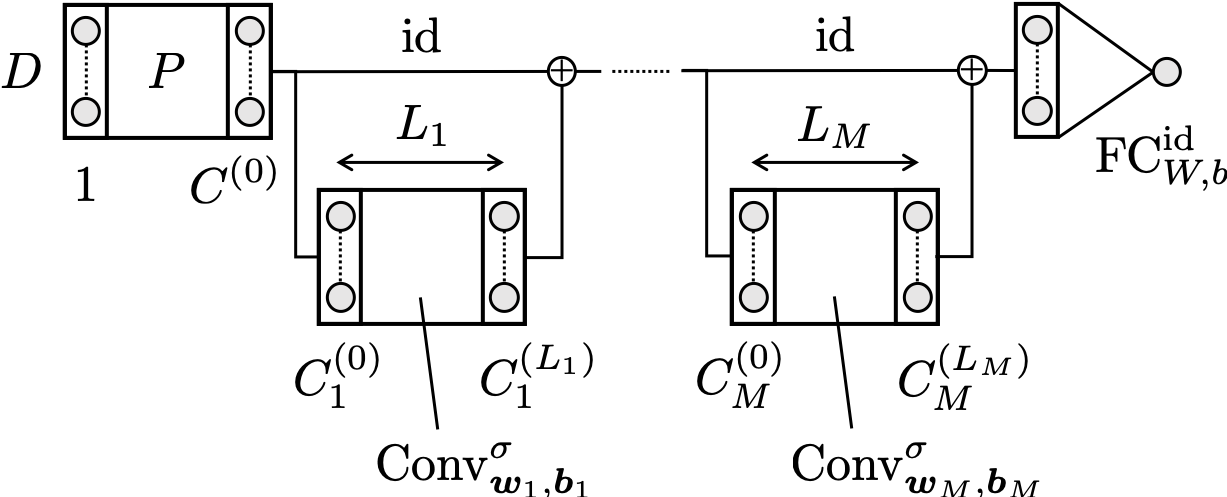}
    \end{center}
    \caption{ResNet-type CNN defined in Definition \ref{def:cnn-def-extend}. Variables are as in Definition \ref{def:cnn-def-extend}.}\label{fig:cnn-def-extend}
\end{figure}
\begin{figure}[t]
    \begin{center}
        \includegraphics[width=0.8\linewidth]{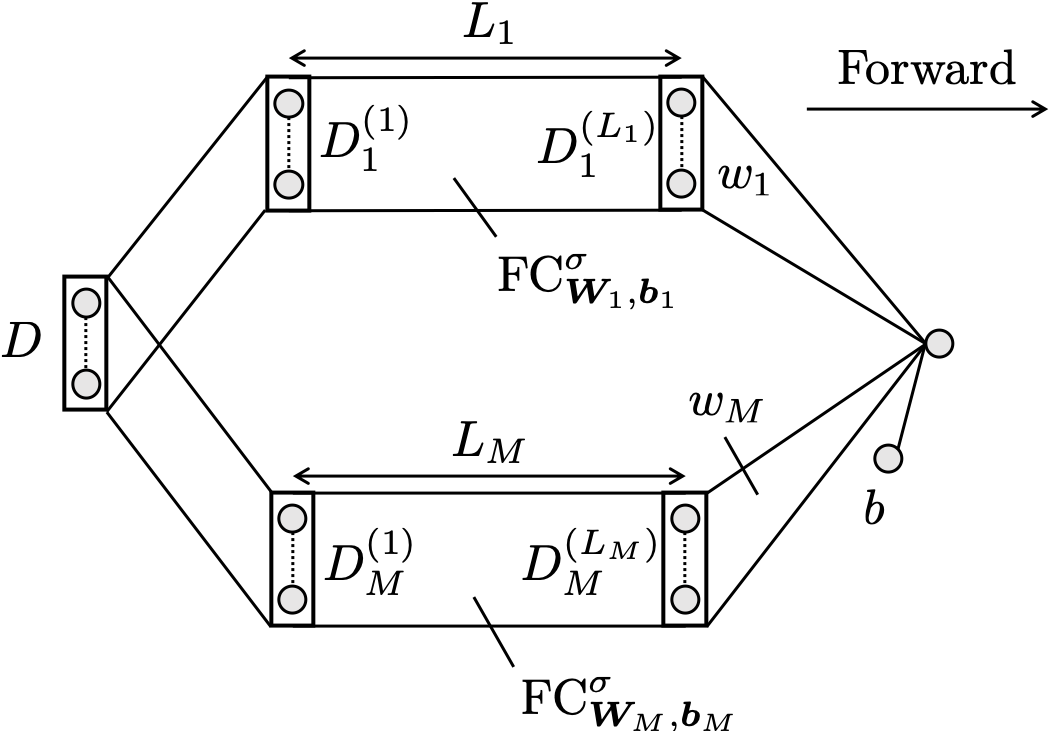}
    \end{center}
    \caption{Schematic view of a block-sparse FNN. Variables are as in Definition \ref{def:block-sparse-fnn-extend}.}\label{fig:block-sparse-fnn-extend}
\end{figure}

\section{Proof of Theorem \ref{thm:fnn-to-cnn}}

We restate Theorem \ref{thm:fnn-to-cnn} in a more general form.
Note that Theorem \ref{thm:fnn-to-cnn} is a special case of Theorem \ref{thm:fnn-to-cnn-extend} where width, depth, channel sizes, and filter sizes are the same among blocks.

\begin{theorem}\label{thm:fnn-to-cnn-extend}
Let $M \in \bbN_{+}$, $K\in \{2, \ldots D\}$, and $L_0 := \left\lceil \frac{D-1}{K-1}\right\rceil$.
Let $L_m, D_m^{(l)}\in \bbN_{+}$ and $\bmD = (D_m^{(l)})_{m, l}$ for $m\in [M]$ and $l\in [L_m]$.
Then, there exist $L'_m\in \bbN_{+}$, $\bmC=(C_m^{(l)})_{m, l}$, and $\bmK=(K_m^{(l)})_{m, l}$ $(m\in [M], l \in [L'_m])$ satisfying the following properties:
\begin{enumerate}
\item $L'_m \leq L_m + L_0$ ($\forall m \in [M]$),
\item $\displaystyle{\max_{l\in[L'_m]} C_m^{(l)} \leq 4\max_{l\in [L_m]} D_m^{(l)}}$ ($\forall m\in [M]$), and
\item $\displaystyle{\max_{l\in [L'_m]} K_m^{(l)} \leq K}$ ($\forall m \in [M]$, $\forall l\in [L'_m]$)
\end{enumerate}
such that for any $B^{\mathrm{(bs)}}, B^{\mathrm{(fin)}} > 0$, any FNN in $\mathcal{F}^{\mathrm{(FNN)}}_{\bmD, B^{\mathrm{(bs)}}, B^{\mathrm{(fin)}}}$ can be realized by a CNN in $\mathcal{F}^{\mathrm{(CNN)}}_{\bmC, \bmK, B^{\mathrm{(conv)}}, B^{\mathrm{(fc)}}}$.
Here, $B^{\mathrm{(conv)}} = \tilde{B}^{\mathrm{(bs)}}$ and $B^{\mathrm{(fc)}} = B^{\mathrm{(fin)}}(1 \vee (\tilde{B}^{\mathrm{(bs)}})^{-1})$, where $\tilde{B}^{\mathrm{(bs)}} = B^{\mathrm{(bs)}} \vee (B^{\mathrm{(bs)}})^{\frac{1}{L_{0}}}$.
Further, if $L_1 = \cdots = L_M$, we can choose $L'_m$ as the same value.
\end{theorem}
\begin{remark}\label{rem:filter-expansion}
For $K \leq K'$, we can embed $\bbR^{K}$ into $\bbR^{K'}$ by inserting zeros: $w=(w_1, \ldots, w_K) \mapsto w'=(w_1, \ldots, w_K, 0, \ldots, 0)$. It is easy to show $L^w = L^{w'}$.
Using this equality, we can expand a size-$K$ filter to size-$K'$.
Furthermore, we can arbitrarily increase the number of output channels of a convolution layer by adding filters consisting of zeros.
Therefore, although properties 2 and 3 allow $C_m^{(l)}$ and $K_m^{(l)}$ to be different values, we can choose $C_m^{(l)}$ and $K_m^{(l)}$ so that inequalities in property 2. and 3. are actually equal by adding filters consisting of zeros.
In particular, when $D^{(l)}_m$'s are same value, we can choose $C_m^{(l)}$ to be same.
\end{remark}

\subsection{Proof Overview}

For $f^{\mathrm{(FNN)}}\in \mathcal{F}^{\mathrm{(FNN)}}$, we realize a CNN $f^{\mathrm{(CNN)}}$ using $M$ residual blocks by ``serializing" blocks in the FNN and converting them into convolution layers.

First, we multiply the channel size by three using the first padding operation.
We will use the first channel for storing the original input signal for feeding to downstream blocks and 
the second and third ones for accumulating properly scaled outputs of each block, that is, $\sum_{m=1}^{m'} w_m^{\top} \mathrm{FC}_{\bmW_m, \bmb_m}^\relu(x)$ where $w_m$ is the weight of the final fully-connected layer corresponding to the $m$-th block.

For $m=1, \ldots, M$, we create the $m$-th residual block from the $m$-th block of $f^{\mathrm{(FNN)}}$.
First, we show that for any $a\in \bbR^D$ and $t\in \bbR$, there exists $L_0$-layered $4$-channel ReLU CNN with $O(D)$ parameters whose first output coordinate equals to a ridge function $x\mapsto (a^{\top} x-t)_+$ (Lemma \ref{lem:realize-ridge-with-cnn} and Lemma \ref{lem:doubling}).
Since the first layer of $m$-th block is the concatenation of $C$ hinge functions, it is realizable by a $4C$-channel ReLU CNN with $L_0$-layers.

For the $l$-th layer of the $m$-th block $(m\in [M], l = 2, \ldots, L'_m)$, we prepare $C$ size-$1$ filters made from the weight parameters of the corresponding layer of the FNN.
Observing that the convolution operation with size-$1$ filter is equivalent to a dimension-wise affine transformation, the first coordinate of the output of $l$-th layer of the CNN is inductively the same as that of the $m$-th block of the FNN.
After computing the $m$-th block FNN using convolutions, we add its output to the accumulating channel in the identity mapping.

Finally, we pick the first coordinate of the accumulating channel and subtract the bias term using the final affine transformation.

\subsection{Decomposition of Affine Transformation}

The following lemma shows that any affine transformation is realizable with a $\left\lceil \frac{D-1}{K-1}\right\rceil$-layered linear conventional CNN (without the final fully-connect layer).

\begin{lemma} \label{lem:realize-ridge-with-cnn}
Let $a\in \bbR^D$, $t\in \bbR$, $K\in \{2, \ldots, D\}$, and $L_0:=\left\lceil \frac{D-1}{K-1}\right\rceil$.
Then, there exists 
\begin{align*}
    w^{(l)}\in
    \begin{cases}
        \bbR^{K \times 2 \times 1} & \text{(for $l=1$)} \\
        \bbR^{K \times 2 \times 2} & \text{(for $l=2, \ldots, L_0-1$)} \\
        \bbR^{K \times 1 \times 2}  & \text{(for $l=L_0$)}
    \end{cases}
\end{align*}
and $b^{(\ell)}\in \bbR$ such that
\begin{enumerate}
    \item
        $\displaystyle{\max_{l\in [L_o]}\|w^{(l)}_m\|_\infty \leq \|a\|_\infty} \vee \|a\|_\infty^{\frac{1}{L_0}}$,
        $\displaystyle{\max_{l\in [L_0]} \|b^{(l)}\|_\infty \leq |t|}$, and
    \item $\displaystyle{\mathrm{Conv}_{\bmw, \bmb}^{\id}: \bbR^{D}\to \bbR^D}$  satisfies $\mathrm{Conv}_{\bmw, \bmb}^{\id}(x) = a^{\top} x - t$ for any $x \in [-1, 1]^D$,
\end{enumerate}
where $\bmw=(w^{(\ell)})_\ell$ and $\bmb=(b^{(\ell)})_\ell$.
\end{lemma}
\begin{proof}
First, we observe that the convolutional layer constructed from $u = \begin{bmatrix} u_1 &  \ldots &  u_K\end{bmatrix}^\top \in \bbR^{K \times 1\times 1}$ takes the inner product with the first $K$ elements of the input signal:
$L^{u}(x) = \sum_{k=1}^{K} u_k x_k$.
In particular, $u = \begin{bmatrix} 0 &  \ldots & 0 &  1\end{bmatrix}^\top \in \bbR^{K \times 1\times 1}$ works as the ``left-translation" by $K-1$.

Let $c = \|a\|_\infty$
We first consider the case $c \geq 1$. We construct $\bmw$ to take the inner product with the $(K-1)$ left-most elements in the first channel and shift the input signal by $(K-1)$ with the second channel. Specifically, we define $\bmw = (w^{(1)}, \ldots, w^{(L_0)})$ by
\begingroup
\allowdisplaybreaks
\begin{alignat*}{23}
    & (w^{(1)})_{:, 1, :} = 
    \begin{bmatrix}
        a_1 \\
        \vdots \\
        a_{K-1} \\
        0
    \end{bmatrix},
    && (w^{(1)})_{:, 2, :} = 
        \begin{bmatrix}
        0      \\
        \vdots \\
        0      \\
        1      
    \end{bmatrix},\\
    & (w^{(l)})_{:, 1, :} = 
    \begin{bmatrix}
        0      & a_{(l-1)(K-1)+1} \\
        \vdots & \vdots \\
        0      & a_{l(K-1)} \\
        0      & 0
    \end{bmatrix},
    && (w^{(l)})_{:, 2, :} = 
    \begin{bmatrix}
        0      & 0 \\
        \vdots & \vdots \\
        0      & 0 \\
        1      & 0
    \end{bmatrix},\\
    & (w^{(L_0)})_{:, 1, :} = 
    \begin{bmatrix}
        0      & a_{(L_0-1)(K-1)+1} \\
        \vdots & \vdots \\
        0      & a_{D} \\
        0      & 0 \\
        \vdots & \vdots \\
        0      & 0 \\        
    \end{bmatrix}. &&
\end{alignat*}
\endgroup
Here, $(w^{(L_0)})_{:, 1, :}$ may not have all-zero rows (this happens when $D = (L_0 - 1)(K-1) + K$, that is, $L_0 = \frac{D-1}{K-1}$.)
We see that
\begin{equation*}
    \max_{l\in [L_o]}\|w_m\|_\infty  =  \|a\|_\infty \vee 1 = \|a\|_\infty.
\end{equation*}
We set $\bmb := (\underbrace{0, \ldots, 0}_{\text{$(L_0-1)$ times}}, t)$.
Then, $\bmw$ and $\bmb$ satisfy conditions 1 and 2.

When $0 < c < 1$, we rescale the elements in $w^{(l)}$'s in the $c \geq 1$ case so that their scales are approximately the same. More specifically, we replace $a_i$ with $a_i c^{-\frac{L_0-1}{L_0}}$ and $1$ with $c^{\frac{1}{L_0}}$. We use the same $\bmb$ as the $c \geq 1$ case.
This change does not change the output of the CNN, thereby satisfying the condition 1.
Since $a_i \leq c$, we have
\begin{equation*}
    \max_{l\in [L_o]}\|w_m\|_\infty  =  c^{\frac{1}{L_0}}.
\end{equation*}
Therefore, the condition 2 is satisfied.
When $c = 0$, we set $w^{(l)} = 0$ and $\bmb$ as in the other cases.
\end{proof}

\subsection{Transformation of a Linear CNN into a ReLU CNN}

The following lemma shows that we can convert any linear CNN to a ReLU CNN with approximately four times larger parameters.
This type of lemma is also found in~\citet{petersen2018optimal} (Lemma 2.3).
The constructed network resembles a CNN with CReLU activation~\citep{shang2016understanding}.

\begin{lemma}\label{lem:doubling}
Let ${\bm C} = (C^{(1)}, \ldots, C^{(L)})\in \bbN^{L}_{+}$ be channel sizes ${\bm K} = (K^{(1)}, \ldots, K^{(L)})\in \bbN^{L}_{+}$ be filter sizes.
Let $w^{(l)}\in \bbR^{K^{(l)}\times C^{l} \times C^{(l)}}$ and $b^{(l)}\in \bbR^{(l)}$.
Consider the linear convolution layers constructed from $\bmw$ and $\bmb$: $f_\id := \mathrm{Conv}^{\id}_{\bmw, \bmb}: \bbR^{D}\to \bbR^{D\times C^{(L)}} \bbN^{L}_{+}$ where $\bmw = (w^{(l)})_l$ and $\bmb = (b^{(l)})_l$ .
Then, there exists a pair $\tilde\bmw = (\tilde{w}^{(l)})_{l\in [L]}, \tilde\bmb = (\tilde{b}^{(l)})_{l\in [L]}$ where $\tilde{w}^{(l)}\in \bbR^{K^{(l)}\times 2C^{(l)}\times 2C^{(l-1)}}$ and $\tilde{b}^{(l)} \in \bbR^{2C^{(l)}}$ such that
\begin{enumerate}
    \item $\displaystyle{\max_{l\in [L]}\|\tilde{w}^{(l)}\|_\infty = \max_{l\in [L]} \|w^{(l)}\|_\infty}$, $\displaystyle{\max_{l\in [L]}\|\tilde{b}^{(l)}\|_\infty = \max_{l\in [L]} \|b^{(l)}\|_\infty}$, and 
    \item $f_\relu := \mathrm{Conv}^{\relu}_{\tilde\bmw, \tilde\bmb}: \bbR^D\to \bbR^{D\times 2C^{(L)}}$, satisfies $f_\relu(\cdot) = (f_{\id}(\cdot)_+, f_{\id}(\cdot)_-)$.
\end{enumerate}
\end{lemma}

\begin{proof}
We define $\tilde{\bmw}$ and $\tilde{\bmb}$ as follows:
\begin{align*}
\dotdot{(\tilde{w}^{(1)})}{k} &=
  \begin{bmatrix} 
    \dotdot{(w^{(1)})}{k}\\
    -\dotdot{(w^{(1)})}{k}
  \end{bmatrix} \ \text{for $k=1, \ldots, K^{(1)}$},\\
\dotdot{(\tilde{w}^{(l)})}{k} &=
\begin{bmatrix} 
    \dotdot{(w^{(l)})}{k} & -\dotdot{(w^{(l)})}{k} \\
    -\dotdot{(w^{(l)})}{k} & \dotdot{(w^{(l)})}{k}
  \end{bmatrix} \ \text{for $k=1, \cdots K^{(l)}$},\\
\tilde{b}^{(l)} &=
  \begin{bmatrix}
    b^{(l)}\\
    -b^{(l)}
  \end{bmatrix}
\end{align*}
By definition, a pair $(\tilde{\bmw}, \tilde{\bmb})$ satisfies the conditions (1) and (2).
For any $x\in \bbR^{D}$, we set $y^{(l)} := \mathrm{Conv}^{\id}_{\bmw[1:l], \bmb[1:l]}(x) \in \bbR^{C^{(l)}\times D}$.
We will prove
\begin{align}
    \mathrm{Conv}^{\relu}_{\tilde{\bmw}[1:l], \tilde{\bmb}[1:l]}(x) = \begin{bmatrix}y^{(l)}_+ & y^{(l)}_-\end{bmatrix}^{\top} \label{eq:relu-conv-induction}
\end{align}
for $l=1, \ldots, L$ by induction.
Note that we obtain $f_\relu (\cdot) = (f_{\id+}(\cdot), f_{\id-}(\cdot))$ by setting $l=L$.
For $l=1$, by definition of $\tilde{w}^{(1)}$ we have,
\begin{align*}
    \dotdot{(\tilde{w}^{(1)})}{\alpha}x^{\beta, :} =
    \begin{bmatrix}
        \dotdot{(w^{(1)})}{\alpha}x^{\beta, :} \\
        - \dotdot{(w^{(1)})}{\alpha}x^{\beta, :}
    \end{bmatrix}
\end{align*}
for any $\alpha, \beta \in [D]$.
Summing them up and using the definition of $\tilde{b}^{(1)}$ yield
\begin{align*}
    & [L^{\tilde{w}^{(1)}}(x) - \bmone_D\otimes\tilde{b}^{(1)}]^{\top}
    = \begin{bmatrix}
        L^{w^{(1)}}(x) - \bmone_D\otimes b^{(1)} \\
        - \left(L^{w^{(1)}}(x) - \bmone_D\otimes b^{(1)}\right)
    \end{bmatrix}^{\top}
\end{align*}
Suppose (\ref{eq:relu-conv-induction}) holds up to $l$ $(l < L)$, by the definition of $\tilde{w}^{(l+1)}$, 
\begin{align*}
    \dotdot{(\tilde{w}^{(l+1)})}{\alpha}
        \begin{bmatrix}
            (y^{(l)}_+)^{\beta, :}\\
            (y^{(l)}_-)^{\beta, :}
        \end{bmatrix}
    &= \begin{bmatrix}
        \dotdot{(w^{(l+1)})}{\alpha} & -\dotdot{(w^{(l+1)})}{\alpha}\\
        -\dotdot{(w^{(l+1)})}{\alpha} & \dotdot{(w^{(l+1)})}{\alpha}
        \end{bmatrix}
    \begin{bmatrix}
        (y^{(l)}_+)^{\beta, :}\\
        (y^{(l)}_-)^{\beta, :}
    \end{bmatrix}\\
    &= \begin{bmatrix}
        \dotdot{(w^{(l+1)})}{\alpha}\left( (y^{(l)}_+)^{\beta, :} - (y^{(l)}_-)^{\beta, :}\right) \\
        -\dotdot{(w^{(l+1)})}{\alpha}\left( (y^{(l)}_+)^{\beta, :} - (y^{(l)}_-)^{\beta, :}\right)
    \end{bmatrix}\\
    &= \begin{bmatrix}
        \dotdot{(w^{(l+1)})}{\alpha} (y^{(l)})^{\beta, :} \\
        -\dotdot{(w^{(l+1)})}{\alpha} (y^{(l)})^{\beta, :}
    \end{bmatrix}
\end{align*}
for any $\alpha, \beta \in [D]$.
Again, by taking the summation and using the definition of $\tilde{b}^{(l+1)}$, we get
\begin{align*}
    [L^{\tilde{w}^{(l+1)}}([y^{(l)}_+, y^{(l)}_-]) - \bmone_D\otimes\tilde{b}^{(1)}]^{\top}
    = \begin{bmatrix}
        L^{w^{(l+1)}}(y^{(l)}) - \bmone_D\otimes b^{(l+1)} \\
        - \left(L^{w^{(l+1)}}(y^{(l)}) - \bmone_D\otimes b^{(l+1)}\right)
    \end{bmatrix}^{\top}.
\end{align*}
By applying ReLU, we get
\begin{align}\label{eq:relu-conv-induction-2}
    \mathrm{Conv}_{\tilde{w}^{(l+1)}, \tilde{b}^{(l+1)}}^{\relu} \left([y^{(l)}_+, y^{(l)}_-]\right)
    = \relu\left([y^{(l+1)}, -y^{(l+1)}]\right).
\end{align}
By using the induction hypothesis, we get
\begin{align*}
    \mathrm{Conv}_{\tilde{\bmw}[1:(l+1)], \tilde{\bmb}[1:(l+1)]}^{\relu}(x)
    &=\mathrm{Conv}_{\tilde{w}^{(l+1)}, \tilde{b}^{(l+1)}}^{\relu} \left([y^{(l)}_+, y^{(l)}_-]\right)\\
    &=\relu\left([y^{(l+1)}, -y^{(l+1)}]\right)\\
    &= [y^{(l+1)}_+, -y^{(l+1)}_-]
\end{align*}
Therefore, the claim holds for $l+1$.
By induction, the claim holds for $L$, which is what we want to prove.
\end{proof}

\subsection{Concatenation of CNNs}

We can concatenate two CNNs with the same depths and filter sizes in parallel.
Although it is almost trivial, we state it formally as a proposition.
In the following proposition, $C^{(0)}$ and ${C'}^{(0)}$ is not necessarily $1$.

\begin{prop}\label{concat2}
Let ${\bm C} = (C^{(l)})_{l\in [L]}$, ${\bm C}' = ({C'}^{(l)})_{l\in [L]}$, and ${\bm K} = (K^{(l)})_{l\in [L]}\in \bbN^{L}_{+}$.
Let $w^{(l)}\in \bbR^{K^{(l)}\times C^{(l)}\times C^{(l-1)}}$, $b\in \bbR^{C^{(l)}}$ and denote $\bmw = (w^{(l)})_l$ and $\bmb = (b^{(l)})_l$.
We define $\bmw'$ and $\bmb'$ in the same way, with the exception that $C^{(l)}$ is replaced with ${C'}^{(l)}$.
We define $\tilde{\bmw} = (\tilde{w}^{(1)}, \ldots, \tilde{w}^{(L)})$ and $\tilde{\bmb} = (\tilde{b}^{(1)}, \ldots, \tilde{b}^{(L)})$ by
\begin{align*}
    \dotdot{(\tilde{w}^{(l)})}{k} &:= 
    \begin{bmatrix}
        w^{(l)} & 0 \\
        0       & {w'}^{(l)}
    \end{bmatrix} \in \bbR^{(C^{(l)} + {C'}^{(l)}) \times (C^{(l-1)} + {C'}^{(l-1)})}\\
    \tilde{b}^{(l)} &:=
    \begin{bmatrix}
        b^{(l)} \\
        {b'}^{(l)}
    \end{bmatrix} \in \bbR^{(C^{(l)} + {C'}^{(l)})}
\end{align*}
for $l\in[L]$ and $k\in[K^{(l)}]$.
Then, we have,
\begin{align*}
    \mathrm{Conv}_{\tilde{\bmw}, \tilde{\bmb}}^{\sigma}(
    \begin{bmatrix}
        x & x'
    \end{bmatrix}) = 
    \begin{bmatrix}
        \mathrm{Conv}_{\bmw, \bmb}^{\sigma}(x) & \mathrm{Conv}_{\bmw', \bmb'}^{\sigma}(x')
    \end{bmatrix}
\end{align*}
for any $x, x'\in \bbR^{D\times C^{(0)}}$ and any $\sigma:\bbR\to \bbR$.
\qed
\end{prop}
Note that by the definition of $\|\cdot\|_0$ and $\|\cdot\|_\infty$, we have
\begin{align*}
    \max_{l\in [L]} \|\tilde{w}^{(l)}\|_{\infty} &= \max_{l\in [L]} \|w^{(l)}\|_{\infty} \vee \|{w'}^{(l)}\|_{\infty}, \quad \text{and} \\
    \max_{l\in [L]} \|\tilde{b}^{(l)}\|_{\infty} &= \max_{l\in [L]} \|b^{(l)}\|_{\infty} \vee \|{b'}^{(l)}\|_{\infty}.
\end{align*}

\subsection{Proof of Theorem \ref{thm:fnn-to-cnn-extend}}

By the definition of $\mathcal{F}^{\mathrm{(FNN)}}_{\bmD, B^{\mathrm{(bs)}}, B^{\mathrm{(fin)}}}$, there exists a 4-tuple $\bmtheta = ((W_m^{(l)})_{m, l}, (b_m^{(l)})_{m, l}, (w_m)_m, b)$ compatible with $(D_m^{(l)})_{m, l}$ ($m\in [M]$ and $l\in [L_m]$) such that
\begin{align*}
    & \max_{m\in[M], l\in [L_m]} (\|W_m^{(l)}\|_{\infty}\vee \|b_m^{(l)}\|_\infty) \leq B^{\mathrm{(bs)}},\\
    & \max_{m\in[M]} \|w_m\|_{\infty}\vee |b| \leq B^{\mathrm{(fin)}},
\end{align*}
and 
$f^{(\mathrm{FNN})} = \mathrm{FNN}_{\bmtheta}^{\relu}$.
We will construct the desired CNN consisting of $M$ residual blocks, whose $m$-th residual block is made from the ingredients of the corresponding $m$-th block in $f^{\mathrm{(FNN)}}$ (specifically, $\bmW_m:=(W_m^{(l)})_{l\in [L_m]}$, $\bmb_m:=(b_m^{(l)})_{l\in [L_m]}$, and $w_m$).

\textbf{[Padding Block]}:
We prepare the padding operation $P$ that multiplies the channel size by 3 (i.e., we set $C^{(0)} = 3$).

\textbf{[$\bm{m=1, \ldots, M}$ Blocks]}: For fixed $m\in [M]$, we first create a CNN realizing $\mathrm{FC}_{\bmW_m, \bmb_m}^{\relu}$.
We treat the first layer (i.e. $l=1$) of $\mathrm{FC}_{\bmW_m, \bmb_m}^{\relu}$ as concatenation of $D^{(1)}_m$ hinge functions $\bbR^D \ni x \mapsto f_d(x) := ((W^{(1)}_{m})_{d} x - b^{(1)}_m)_+$ for $d\in [D^{(1)}_m]$.
Here, $(W^{(1)}_{m})_{d}\in \bbR^{1\times D}$ is the $d$-th row of the matrix $W^{(1)}_{m}\in \bbR^{D^{(1)}_m\times D}$.
We apply Lemma \ref{lem:realize-ridge-with-cnn} and Lemma \ref{lem:doubling} and obtain ReLU CNNs realizing the hinge functions.
By combining them in parallel using Proposition \ref{concat2},
we have a learnable parameter $\bmtheta^{(1)}_m$ such that the ReLU CNN $\mathrm{Conv}^{\relu}_{\bmtheta^{(1)}_m}:\bbR^{D\times 2}\to \bbR^{D\times 2D^{(1)}_m}$ constructed from $\bmtheta^{(1)}_m$ satisfies
\begin{align*}
    \mathrm{Conv}^{\relu}_{\bmtheta^{(1)}_m}(\begin{bmatrix}x & x'\end{bmatrix}^\top)_1 = \begin{bmatrix}f_{1}(x) & \ast & \cdots     & f_{D^{(1)}_m}(x) & \ast \end{bmatrix}^\top.
\end{align*}
Since we double the channel size in the $m=0$ part, the identity mapping has two channels.
Therefore, we made $\mathrm{Conv}^{\relu}_{\bmtheta^{(1)}_m}$ so that it has two input channels and neglects the input signals coming from the second one.
This is possible by adding filters consisting of zeros appropriately.

Next, for $l$-th layer ($l = 2, \ldots, L_m$), we prepare size-$1$ filters
$w_m^{(2)} \in \bbR^{1\times D^{(2)}_m \times 2D^{(1)}}_m$ for $l=2$ and
$w_m^{(l)} \in \bbR^{1\times D^{(l)}_m \times 2D^{(l-1)}_m}$ for $l = 3, \ldots, D^{(L_m)}_m$ defined by
\begin{align*}
    \dotdot{(w^{(l)}_m)}{1}:= \begin{cases}
        W^{(2)}_m \otimes \begin{bmatrix}1 & 0\end{bmatrix} & \text{if $l=2$} \\
        W^{(l)}_m & \text{if $l=3, \ldots, D^{(L_m)}_m$},
    \end{cases}
\end{align*}
where $\otimes$ is the Kronecker product of matrices.
Intuitively, the $l=2$ layer will pick all odd indices of the output of $\mathrm{Conv}^{\relu}_{\bmtheta_{m}^{(1)}}$ and apply the fully-connected layer.
We construct CNNs from $\theta^{(l)}_m := (w^{(l)}_m, b^{(l)}_m)$ ($l\geq 2$) and concatenate them along with $\mathrm{Conv}_{\bmtheta^{(1)}_m}^{\relu}$:
\begin{align*}
    \mathrm{Conv}_{m}:=
        \mathrm{Conv}_{\theta^{(L_m)}_m}^{\relu} \circ
        \cdots \circ
        \mathrm{Conv}_{\theta^{(2)}_m}^{\relu} \circ
        \mathrm{Conv}_{\bmtheta^{(1)}_m}^{\relu}.
\end{align*}
Note that $\mathrm{Conv}^{\relu}_{\theta_m^{(l)}}$ ($l\geq 2$) just rearranges parameters of $\mathrm{FC}^{\relu}_{\bmW_m, \bmb_m}$.
The output dimension of $\mathrm{Conv}_m$ is either $\bbR^{D\times 2D^{(L_m)}_m}$ (if $L_m = 1$) or $\bbR^{D\times D^{(L_m)}_m}$ (if $L_m \geq 2$).,
We denote the output channel size (either $2D^{(L_m)}_m$ or $D^{(L_m)}_m$) by $D_m^{\mathrm{(out)}}$.
By the inductive calculation, we have
\begin{align*}
    \mathrm{Conv}_m (x)_1 =
    \begin{cases}
        \mathrm{FC}_{\bmW_m, \bmb_m}^{\relu} (x) \otimes \begin{bmatrix} 1 & 0\end{bmatrix} & \text{if $L_m = 1$}\\
        \mathrm{FC}_{\bmW_m, \bmb_m}^{\relu} (x) & \text{if $L_m \geq 2$}
    \end{cases}.
\end{align*}
By definition, $\mathrm{Conv}_{m}$ has $L_0 + L_m - 1$ layers and at most
$4D_m^{(1)} \vee \max_{l=2, \ldots L_m} D_m^{(l)} \leq 4 \max_{l\in [L_m]} D_m^{(l)}$ channels.
The $\infty$-norm of its parameters does not exceed that of parameters in $\mathrm{FC}^{\relu}_{\bmW_m, \bmb_m}$.

Next, we consider the filter $\tilde{w}_m\in \bbR^{1\times 3 \times D_m^{\mathrm{(out)}}}$ defined by
\begin{align*}
    (\tilde{w}_m)_{1, :, :} =
    \frac{\tilde{B}^{\mathrm{(bs)}}}{B^{\mathrm{(fin)}}}\begin{cases}
        \begin{bmatrix}
            0& \cdots & 0 \\ 
            \multicolumn{3}{c}{w_m \otimes \begin{bmatrix}0 & 1\end{bmatrix}} \\
            \multicolumn{3}{c}{-w_m \otimes \begin{bmatrix}0 & 1\end{bmatrix}}
        \end{bmatrix} &  \text{if $L_m = 1$}\\
        \begin{bmatrix}
            0&\cdots & 0 \\
            \multicolumn{3}{c}{w_m} \\
            \multicolumn{3}{c}{-w_m} \\        
        \end{bmatrix}& \text{if $L_m \geq 2$},
    \end{cases}
\end{align*}
where, $\tilde{B}^{\mathrm{(bs)}} = B^\mathrm{(bs)} \vee (B^\mathrm{(bs)})^{\frac{1}{L_0}}$.
Then, $\mathrm{Conv}'_m:=\mathrm{Conv}_{\tilde{w}_m, 0}^{\relu}$ adds the output of $m$-th residual block, weighted by $w_m$, to the second channel in the identity connections, while keeping the first channel intact.
Note that the final layer of each residual block does not have the ReLU activation.
By definition, $\mathrm{Conv}'_m$ has $D^{(L_m)}_m$ parameters.

Given $\mathrm{Conv}_{m}$ and $\mathrm{Conv}'_m$ for each $m\in [M]$, we construct a CNN realizing $\mathrm{FNN}^{\relu}_{\bmtheta}$.
Let $f^{\mathrm{(conv)}}: \bbR^D\to \bbR^{D\times 3}$ be the sequential interleaving concatenation of $\mathrm{Conv}_{m}$ and $\mathrm{Conv}'_m$, that is,
\begin{align*}
    f^{\mathrm{(conv)}}
    := (\mathrm{Conv}'_M \circ \mathrm{Conv}_{M} + I) \circ \cdots
    \quad \circ (\mathrm{Conv}'_1 \circ \mathrm{Conv}_{1} + I) 
    \circ P.
\end{align*}
Then, we have
\begin{align*}
    f^{\mathrm{(conv)}}_{1, :} = \begin{bmatrix} 0 & z_1 & z_2 \end{bmatrix}\in \bbR^{3}
\end{align*}
where $z_1 = \frac{\tilde{B}^{\mathrm{(bs)}}}{B^{\mathrm{(fin)}}} \sum_{m=1}^{M} \left(w_m^\top \mathrm{FC}^{\relu}_{\bmW_m, \bmb_m}\right)_+$ and
$z_2 = \frac{\tilde{B}^{\mathrm{(bs)}}}{B^{\mathrm{(fin)}}} \sum_{m=1}^{M} \left(w_m^\top \mathrm{FC}^{\relu}_{\bmW_m, \bmb_m}\right)_-$.

\textbf{[Final Fully-connected Layer]} 
Finally, we set
\begin{align*}
    w:=\frac{B^{\mathrm{(fin)}}}{\tilde{B}^{\mathrm{(bs)}}} \begin{bmatrix}
        0 & 0 & \cdots & 0 \\
        1 & 0 & \cdots & 0 \\
        -1 & 0 & \cdots & 0 \\
    \end{bmatrix}\in \bbR^{D \times 3}
\end{align*}
and put $\mathrm{FC}^{\id}_{w, b}$ on top of $f^{\mathrm{(conv)}}$ to pick the first coordinate of $f^{\mathrm{(conv)}}$ and subtract the bias term.
By definition, $f^{\mathrm{(CNN)}} := \mathrm{FC}^{\id}_{w, b} \circ f^{\mathrm{(conv)}}$ satisfies $f^{\mathrm{(CNN)}} = f^{\mathrm{(FNN)}}$.

\textbf{[Property Check]} We check $f^{\mathrm{(FNN)}}$ satisfies the desired properties:

\textbf{Property 1}: Since $\mathrm{Conv}_m$ and $\mathrm{Conv}'_m$ has $L_0 + L_m - 1$ and $1$ layers, respectively, the $m(\geq 1)$-th residual block of $f^{\mathrm{(CNN)}}$ has $L'_m = L_0 + L_m$ layers.
In particular, if $L_m$'s are the same, we can choose $L'_m$ as the same value $L_0 + L_m$.

\textbf{Property 2}: $\mathrm{Conv}_m$ has at most $4 \max_{l\in [L_m]} D_m^{(l)}$ channels and $\mathrm{Conv}'_m$ has at most $2$ channels, respectively.
Therefore, the channel size of the $m$-th block is at most $4 \max_{l\in [L_m]} D_m^{(l)}$.

\textbf{Property 3}: Since each filter of $\mathrm{Conv}_{m}$ and $\mathrm{Conv}'_m$ is at most $K$, the filter size of $\mathrm{CNN}$ is also at most $K$.

\textbf{Properties on $B^{\mathrm{(conv)}}$ and $B^{\mathrm{(fin)}}$}:
Parameters of $f^{\mathrm{(conv)}}$ are either $0$, or parameters of $\mathrm{FC}^{\relu}_{\bmW_m, \bmb_m}$, whose absolute value is bounded by $B^{\mathrm{(bs)}}$ or $\frac{\tilde{B}^{\mathrm{(bs)}}}{B^{\mathrm{(fin)}}}w_m$.
Since we have $\|w_m\|_\infty \leq B^{\mathrm{(fin)}}$, the $\infty$-norm of parameters in $f^\mathrm{(CNN)}$ is bounded by $\tilde{B}^{\mathrm{(bs)}}$.
The parameters of the final fully-connected layer $\mathrm{FC}_{w, b}$ is either $\frac{B^{\mathrm{(fin)}}}{\tilde{B}^{\mathrm{(bs)}}}$, $0$, or $b$, therefore their norm is bounded by $\frac{B^{\mathrm{(fin)}}}{\tilde{B}^{\mathrm{(bs)}}}\vee B^{\mathrm{(fin)}}$.
\qed

As discussed at the beginning of this section, Theorem \ref{thm:fnn-to-cnn} is the special case of Theorem \ref{thm:fnn-to-cnn-extend}.

\begin{remark}
Another way to construct a CNN identical (as a function) to a given FNN is as follows.
First, we use a ``rotation" convolution with $D$ filters, each of which has a size $D$, to serialize all input signals to channels of a single input dimension.
Then, apply size-1 convolution layers, whose $l$-th layer consists of appropriately arranged weight parameters of the $l$-th layer of the FNN.
This is essentially what~\citet{petersen2018equivalence} did to prove the existence of a CNN equivalent to a given FNN.
To restrict the size of filters to $K$, we should further replace the first convolution layer with $O(D/K)$ convolution layers with size-$K$ filters.
We can show essentially the same statement using this construction method.
\end{remark}

\section{Proof of Theorem \ref{thm:estimation}}

Same as Theorem \ref{thm:fnn-to-cnn}, we restate Theorem \ref{thm:estimation} in a more general form.
We denote
$\mathcal{F}^{\mathrm{(CNN)}} := \mathcal{F}^{\mathrm{(CNN)}}_{\bmC, \bmK, B^{\mathrm{(conv)}}, B^{\mathrm{(fc)}}}$
and
$\mathcal{F}^{\mathrm{(FNN)}}:=\mathcal{F}^{\mathrm{(FNN)}}_{\bmD, B^{\mathrm{(bs)}}, B^{\mathrm{(fin)}}}$ in shorthand.

\begin{theorem}\label{thm:estimation-extend}
Let $f^\circ: \bbR^{D}\to \bbR$ be a measurable function and $B^{\mathrm{(bs)}}, B^{\mathrm{(fin)}} > 0$.
Let $M$, $K$, $L_0$, $L_m$, and $\bmD$ as in Theorem \ref{thm:fnn-to-cnn-extend}.
Suppose $L'_m, \bmC, \bmK, B^{\mathrm{(conv)}}$ and $B^{\mathrm{(fc)}}$ satisfy  
$\mathcal{F}^{\mathrm{(FNN)}} \subset \mathcal{F}^{\mathrm{(CNN)}}$
for $B^{\mathrm{(bs)}}$ and $B^{\mathrm{(fin)}}$ (their existence is ensured for any $B^{\mathrm{(bs)}}$ and $B^{\mathrm{(fin)}}$ by Theorem \ref{thm:fnn-to-cnn-extend}).
Suppose that the covering number of $\mathcal{F}^{\mathrm{(CNN)}}$ is larger than $3$.
Then, the clipped ERM estimator $\hat{f}$ in
$\mathcal{F}:= \{ \mathrm{clip}[f] \mid f \in \mathcal{F}^{\mathrm{(CNN)}} \}$ satisfies
\begin{align}\label{eq:estimation-bound-2}
    \bbE_{\mathcal{D}} \|\hat{f} - f^\circ\|_{\mathcal{L}^2(\mathcal{P}_X)}^2
    \leq C \left(\inf_{f} \|f-f^\circ\|^2_{\infty} + \frac{\tilde{F}^2}{N} \Lambda_2\log (2\Lambda_1BN) \right).
\end{align}
Here, $f$ ranges over $\mathcal{F}^{\mathrm{(FNN)}}$, $C_0>0$ is a universal constant, $\tilde{F} := \frac{\|f^\circ\|_\infty}{\sigma} \vee \frac{1}{2}$, and $B = B^{\mathrm{(conv)}}\vee B^{\mathrm{(fc)}}$. $\Lambda_1 = \Lambda_1(\mathcal{F}^{\mathrm{(CNN)}})$ and $\Lambda_2 = \Lambda_2(\mathcal{F}^{\mathrm{(CNN)}})$ are defined by
\begin{align*}
    \Lambda_1 &:= (2M+3)C^{(0)}D (1\vee B^{\mathrm{(fc)}}) (1\vee B^\mathrm{(conv)}) \varrho \varrho^+ \\
    \Lambda_2 &:= \sum_{m=1}^{M} \sum_{l=1}^{L'_m} \left(C^{(l-1)}_m C^{(l)}_m K^{(l)}_m+C^{(l)}_m\right) + C^{(0)} D + 1,
\end{align*}
where $\varrho=\prod_{m=1}^M (1+\rho_m)$,
$\varrho^+=1 + \sum_{m=1}^{M}L'_m\rho_{m}^+$,
$\rho_m := \prod_{l=1}^{L'_m} C^{(l-1)}_mK^{(l)}_m B^{\mathrm{(conv)}}$ and $\rho_{m}^+ := \prod_{l=1}^{L'_m} (1\vee C^{(l-1)}_mK^{(l)}_m B^{\mathrm{(conv)}})$.
\end{theorem}

Again, Theorem \ref{thm:estimation} is a special case of Theorem \ref{thm:estimation-extend} where width, depth, channel sizes, and filter sizes are the same among blocks.
Note that the definitions of $\Lambda_1$, $\Lambda_2$, $\rho$, $\rho^+$, $\varrho$, and $\varrho^+$ in Theorem \ref{thm:estimation} and Theorem \ref{thm:estimation-extend}  are consistent by this specialization.

\subsection{Proof Overview}

We relate the approximation error of Theorem \ref{cor:approximation-barron} with the estimation error using the covering number of the hypothesis class $\mathcal{F}^{\mathrm{(CNN)}}$.
Although there are several theorems of this type, we employ the one in~\citet{schmidt2017nonparametric} due to its convenient form (Lemma \ref{lem:schmidt-hieber-lem10}).
We can prove that the logarithm of the covering number is upper bounded by $\Lambda_2\log ((B^{\mathrm{(conv)}}\vee B^{\mathrm{(fc)}})\Lambda_1/\varepsilon)$ (Lemma \ref{lem:covering-number}) using the similar techniques to the one in~\citet{schmidt2017nonparametric}.
Theorem \ref{thm:estimation} is the immediate consequence of these two lemmas.

To prove Corollary \ref{cor:estimation-error-wrt-sample-size}, we set $M = O(N^{\alpha})$ for some $\alpha > 0$.
Then, under the assumption of the corollary, we have
$\|f^\circ - \hat{f}\|_{\mathcal{L}^2(\mathcal{P}_x)}^2 = \tilde{O}\left(\max \left( N^{-2\alpha\gamma_1}, N^{\alpha\gamma_2-1} \right)\right)$ from Theorem \ref{thm:estimation}.
The order of the right-hand side with respect to $N$ is minimized when $\alpha = \frac{1}{2\gamma_1+\gamma_2}$.
By substituting $\alpha$, we can prove Corollary \ref{cor:estimation-error-wrt-sample-size}.

\subsection{Covering Number of CNNs}\label{sec:covering-number-cnns}

This section aims to prove Lemma \ref{lem:covering-number}, stated in Section \ref{sec:covering-number-bound}, that evaluates the covering number of the set of functions realized by CNNs.

\subsubsection{Bounds for convolutional layers}\label{sec:bounds-for-conv-layers}

We assume $w, w'\in \bbR^{K\times J\times I}$, $b, b'\in \bbR$, and $x\in \bbR^{D\times I}$ unless specified.
We have in mind that the activation function $\sigma$ is either the ReLU function or the identity function $\id$.
But the following proposition holds for any $1$-Lipschitz function such that $\sigma(0) = 0$.
Remember that we can treat $L^{w}$ as a linear operator from $\bbR^{D\times I}$ to $\bbR^{D\times J}$.
We endow $\bbR^{D\times I}$ and $\bbR^{D\times J}$ with the sup norm and denote the operator norm $L^{w}$ by $\|L^{w}\|_\op$.

\begin{prop}\label{prop:conv-op-bound}
It holds that $\|L^{w}\|_\mathrm{op} \leq IK\|w\|_\infty$.
\end{prop}
\begin{proof}
Write $w = (w_{kji})_{k\in [K], j\in[J], i\in[I]}$,
$L^{w} = ((L^{w})^{\beta, j}_{\alpha, i})_{\alpha, \beta\in[D] , j\in[J], i\in[I]}$. 
For any $x=(x_{\alpha, i})_{\alpha\in [D], i\in[I]}\in \bbR^{D\times I}$, the sup norm of $y:=(y_{\beta j})_{\beta\in [D]j\in[J]} = L^{w}(x)$ is evaluated as follows:
\begingroup
\allowdisplaybreaks
\begin{align*}
\|y\|_\infty
= \max_{\beta, j} |y_{\beta, j}|
& \leq \max_{\beta, j} \sum_{\alpha, i} |(L^{w})^{\beta, j}_{\alpha, i}| |x_{\alpha, i}|\\
& \leq \max_{\beta, j} \sum_{\alpha, i} |(L^{w})^{\beta, j}_{\alpha, i}| \|x\|_\infty\\
&= \max_{\beta, j} \sum_{\alpha, i} |{w}_{(\alpha - \beta + 1), j, i}| \|x\|_\infty\\
&\leq IK \|w\|_\infty \|x\|_\infty
\end{align*}
\endgroup
\end{proof}

\begin{prop}\label{prop:conv-sup-bound}
It holds that $\|\mathrm{Conv}^{\sigma}_{w, b}(x)\|_\infty \leq \|L^w\|_\mathrm{op} \|x\|_\infty + |b|$.
\end{prop}
\begin{proof}
\begin{align*}
\|\mathrm{Conv}^{\sigma}_{w, b}(x)\|_\infty
&\leq \|\sigma(L^{w}(x)-\bmone_D \otimes b)\|_\infty\\
&\leq \|L^{w}(x)-\bmone_D \otimes b\|_\infty\\
&\leq \|L^{w}(x)\|_\infty + \|\bmone_D \otimes b\|_\infty\\
&\leq \|L^{w}\|_\op \|x\|_\infty + |b|.
\end{align*}
\end{proof}

\begin{prop}\label{prop:conv-lipchitz-bound}
The Lipschitz constant of $\mathrm{Conv}^{\sigma}_{w, b}$ is bounded by $\|L^w\|_\mathrm{op}$.
\end{prop}
\begin{proof}
For any $x, x' \in \bbR^{D\times I}$,
\begin{align*}
    \|\mathrm{Conv}_{w, b}^{\sigma}(x) - \mathrm{Conv}_{w, b}^{\sigma}(x')\|_\infty
    &= \|\sigma\left(L^w(x) - \bmone_D\otimes b\right) - \sigma\left(L^w(x') - \bmone_D\otimes b\right) \|_\infty\\
    &\leq \|\left(L^{w}(x)-\bmone_D\otimes b\right) - \left(L^{w}(x')-\bmone_D\otimes b\right)\|_\infty\\
    &\leq \|L^{w}(x-x')\|_\infty\\
    &\leq \|L^{w}\|_\op \|x-x'\|_\infty.
\end{align*}
Note that the first inequality holds because the ReLU function is $1$-Lipschitz.
\end{proof}

\begin{prop}\label{prop:conv-diff-bound}
It holds that $\|\mathrm{Conv}_{w, b}^{\sigma}(x) - \mathrm{Conv}_{w', b'}^{\sigma}(x)\|\leq \|L^{w-w'}\|_\mathrm{op}\|x\|_\infty + |b-b'|$.
\end{prop}
\begin{proof}
\begin{align*}
\|\mathrm{Conv}_{w, b}^{\sigma}(x) - \mathrm{Conv}_{w', b'}^{\sigma}(x)\|_\infty
&= \|\sigma(L^w(x) - \bmone_D\otimes b) - \sigma(L^{w'}(x) - \bmone_D\otimes b') \|_\infty\\
&\leq \|(L^w(x) - \bmone_D\otimes b) - (L^{w'}(x) - \bmone_D\otimes b')\|\\
&= \|L^w(x)-L^{w'}(x)\| + \|\bmone_D\otimes (b-b')\|_\infty\\
&\leq \|L^{w-w'}\|_\mathrm{op}\|x\|_\infty + |b-b'|
\end{align*}
\end{proof}

\subsubsection{Bounds for fully-connected layers}

In the following propositions in this subsection, we assume $W, W'\in \bbR^{DC\times C'}$, $b, b'\in \bbR^{C'}$, and $x\in \bbR^{D\times C}$.
Again, these propositions hold for any $1$-Lipschitz function $\sigma:\bbR\to \bbR$ such that $\sigma(0)=0$.
But $\sigma = \relu$ or $\id$ is enough for us.

\begin{prop}\label{prop:fc-sup-bound}
It holds that $\|\mathrm{FC}^{\sigma}_{W, b}(x)\|_\infty \leq \|W\|_0 \|W\|_\infty\|x\|_\infty + \|b\|_\infty$.
\end{prop}
\begin{proof}
\begin{align*}
\|\mathrm{FC}^{\sigma}_{W, b}(x)\|_\infty
\leq \|W\vect{x} - b\|_\infty
\leq \|W\vect{x}\|_\infty + \|b\|_\infty
\leq \max_{j} \sum_{\alpha, i} \left|W_{\alpha, i, j} x^{\alpha, i}\right| + \|b\|_\infty.
\end{align*}
The number of non-zero summands in the summation is at most $\|W\|_0$
and each summand is bounded by $\|W\|_\infty \|x\|_\infty$
Therefore, we have
$\|\mathrm{FC}^{\sigma}_{W, b}(x)\|_\infty \leq \|W\|_0 \|W\|_\infty\|x\|_\infty + \|b\|_\infty$.
\end{proof}

\begin{prop}\label{prop:fc-lipschitz-bound}
The Lipschitz constant of $\mathrm{FC}^{\sigma}_{W, b}$ is bounded by $\|W\|_0\|W\|_\infty$.
\end{prop}
\begin{proof}
For any $x, x'\in \bbR^{D\times C}$,
\begin{align*}
\|\mathrm{FC}^{\sigma}_{W, b}(x) - \mathrm{FC}^{\sigma}_{W, b}(x')\|_\infty
& \leq \|(W\vect{x} - b) - (W\vect{x'} - b)\|_\infty\\
& \leq \|W (\vect{x}-\vect{x'})\|\infty\\
& \leq \|W\|_0 \|W\|_\infty \|\vect{x}-\vect{x'}\|_\infty.
\end{align*}
\end{proof}

\begin{prop}\label{prop:FC-diff-bound}
It holds that $\|\mathrm{FC}^{\sigma}_{W, b}(x) - \mathrm{FC}^{\sigma}_{W', b'}(x)\|_\infty \leq (\|W\|_0 + \|W'\|_0)\|W-W'\|_\infty\|x\|_\infty + \|b-b'\|_\infty$.
\end{prop}
\begin{proof}
\begin{align*}
\|\mathrm{FC}^{\sigma}_{W, b}(x) - \mathrm{FC}^{\sigma}_{W', b'}(x)\|_\infty
&\leq \|(W\vect{x} - b) - (W'\vect{x} - b')\|_\infty\\
&= \|((W-W')\vect{x} - (b - b')\|_\infty\\
&\leq \|(W-W')\vect{x}| + \|b - b'\|_\infty\\
&\leq \|W-W'\|_0 \|W-W'\|_\infty \|x\|_\infty + \|b - b'\|_\infty\\
&\leq (\|W\|_0 + \|W'\|_0) \|W-W'\|_\infty\|x\|_\infty + \|b-b'\|_\infty
\end{align*}
\end{proof}

\subsubsection{Bounds for residual blocks} \label{sec:bounds-for-residual-block}

In this section, we denote the architecture of CNNs by $\bmC=(C^{(l)})_{l\in [L]} \in \bbN^{L}_{+}$ and $\bmK=(K^{(l)})_{l\in [L]}\in \bbN^{L}_{+}$ and the norm constraint on the convolution part by $B^{\mathrm{(conv)}}$ ($C^{(0)}$ need not equal to $1$ in this section). 
Let $w^{(l)}, {w'}^{(l)} \in \bbR^{K^{(l)}\times C^{(l)}\times C^{(l-1)}}$
and $b^{(l)}, {b'}^{(l)} \in \bbR$.
We denote $\bmw := (w^{(l)})_{l\in [L]}$, $\bmb := (b^{(l)})_{l\in [L]}$, $\bmw' := ({w'}^{(l)})_{l\in [L]}$, and $\bmb := (b^{(l)})_{l\in [L]}$.

For $1\leq l \leq l'\leq L$, we denote $\rho(l, l'):= \prod_{i=l}^{l'} (C^{(i-1)}K^{(i)}B^{\mathrm{(conv)}})$
and $\rho^+(l, l'):= \prod_{i=l}^{l'} 1\vee (C^{(i-1)}K^{(i)}B^{\mathrm{(conv)}})$.

\begin{prop}\label{prop:conv-sup-bound-multiple-layer}
Let $l \in [L]$.
We assume $\max_{l\in [L]} \|w^{(l)}\|_\infty \vee \|b^{(l)}\|_\infty \leq B^{\mathrm{(conv)}}$.
Then, for any $x\in [-1, 1]^{D\times C^{(0)}}$,
we have $\|\mathrm{Conv}_{\bmw[1:l], \bmb[1:l]}^{\sigma} (x)\|_\infty
\leq \rho(1, l)\|x\|_\infty + B^{\mathrm{(conv)}}l\rho^+(1, l)$.
\end{prop}
\begin{proof}
We write in shorthand as $C_{[s:t]}:= \mathrm{Conv}_{\bmw[s:t], \bmb[s:t]}^{\sigma}$.
Using Proposition \ref{prop:conv-sup-bound} recursively, we get
\begin{align*}
    \|C_{[1:l]}(x)\|_\infty
    & \leq \|L^{w^{(l)}}\|_\op\|C_{[1:l-1]}(x)\|_\infty + \|b^{(l)}\|_\infty \\
    & \ldots\\
    & \leq \|x\|_\infty\prod_{i=1}^{l} \|L^{w^{(i)}}\|_\op
        + \sum_{i=2}^{l} \|b^{(i-1)}\|_\infty \prod_{j=i}^{l} \|L^{w^{(j)}}\|_\op
        + \|b^{(l)}\|_\infty.
\end{align*}
By Proposition \ref{prop:conv-op-bound} and assumptions $\|w^{(i)}\|_\infty \leq B^{\mathrm{(conv)}}$ and $\|b^{(i)}\|_\infty \leq B^{\mathrm{(conv)}}$, it is further bounded by
\begin{align*}
& \|x\|_\infty\prod_{i=1}^{l} (C^{(i-1)}K^{(i)}B^{\mathrm{(conv)}}) 
    + B^{\mathrm{(conv)}} \sum_{i=2}^{l} \prod_{j=i}^{l} (C^{(j-1)}K^{(j)}B^{\mathrm{(conv)}})
    + B^{\mathrm{(conv)}}\\
& \leq \rho(1, l)\|x\|_\infty + B^{\mathrm{(conv)}}l\rho^+(1, l)
\end{align*}
\end{proof}

\begin{prop}\label{prop:diff-res-block}
Let $\varepsilon>0$, suppose $\max_{l\in [L]} \|w^{(l)} - {w'}^{(l)}\|_\infty \leq \varepsilon$ and
$\max_{l\in [L]} \|b^{(l)} - {b'}^{(l)}\|_\infty \leq \varepsilon$, then $\|C_{[1:L]} - C'_{[1:L]}(x)\|_\infty \leq (L\rho(1, L)\|x\|_\infty + (1\vee B^{\mathrm{(conv)}})L^2 \rho^+(1, L))\varepsilon$
for any $x\in \bbR^{D\times C^{(0)}}$.
\end{prop}
\begin{proof}
For any $l\in [L]$, we have
\begingroup
\allowdisplaybreaks
\begin{align}
& \left|C'_{[l+1:L]} \circ (C_l - C'_l)
    \circ C_{[1:l-1]} (x) \right| \nonumber \\
& \leq\|C'_{[l+1:L]} \circ (C_l - C'_l)\circ C_{[1:l-1]}(x)\|_\infty \nonumber \\
& \leq \rho(l+1, L)\left\|(C_l - C'_l)\circ C_{[1:l-1]}(x) \right\|_\infty \nonumber 
    \quad \text{(by Proposition \ref{prop:conv-op-bound} and \ref{prop:conv-lipchitz-bound})}\nonumber \\
& \leq \rho(l+1, L) \left(\rho(l, l) \|C_{[1:l-1]}\|_\infty \varepsilon+ \varepsilon \right) \nonumber 
    \quad \text{(by Proposition \ref{prop:conv-op-bound} and \ref{prop:conv-diff-bound})}\nonumber \\
& \leq \rho(l+1, L)\left(\rho(l, l)(\rho(1, l-1)\|x\|_\infty + B^{\mathrm{(conv)}}(l-1)\rho_+(1, l-1))+ 1 \right)\varepsilon \nonumber \\
& \quad \text{(by Proposition \ref{prop:conv-sup-bound-multiple-layer})}\nonumber \\
& = \left(\rho(1, L)\|x\|_\infty + (1\vee B^{\mathrm{(conv)}})L\rho_+(1, L)\right)\varepsilon
\end{align}
\endgroup

Therefore,
\begin{align*}
    \|C_{[1:L]} (x) - C'_{[1:L]}(x)\|_\infty 
    & \leq \sum_{l=1}^{L} \|C_{[l+1:L]} \circ (C_l - C'_l) \circ C_{[1:l-1]} (x)\|_\infty\\
    & \leq (L\rho(1, L)\|x\|_\infty + (1\vee B^\mathrm{(conv)})L^2 \rho^+(1, L))\varepsilon
\end{align*}
\end{proof}

\subsubsection{Putting them all}

Let $M, L_m, C_m^{(l)}, K_{m}^{(l)}\in \bbN_{+}$, $\bmC:=(C_m^{(l)})_{m, l}$, and $\bmK:=(K_m^{(l)})_{m, l}$ for $m\in [M]$ and $l\in [L_m]$.
Let $\bmtheta=((w_{m}^{(l)})_{m, l}, (b_{m}^{(l)})_{m, l}, W, b)$ and $\bmtheta'=(({w'}_{m}^{(l)})_{m, l}, ({b'}_{m}^{(l)})_{m, l}, W', b')$ be tuples compatible with $(\bmC, \bmK)$
such that $\mathrm{CNN}_{\bmtheta}^{\relu}$, $\mathrm{CNN}_{\bmtheta'}^{\relu}$ $\in \mathcal{F}^{\mathrm{(CNN)}}_{{\bm C}, {\bm K}, B^{\mathrm{(conv)}}, B^{\mathrm{(fc)}}}$
for some $B^{\mathrm{(conv)}}, B^{\mathrm{(fc)}} > 0$.
We denote the $l$-th convolution layer of the $m$-th block by $C_{m}^{(l)}$ and the $m$-th residual block of
by $C_m$:
\begin{align*}
    C_{m}^{(l)} &:= \begin{cases}
        \mathrm{Conv}_{w_m^{(l)}}^{\id} & \text{(if $l = L_m$)}\\
        \mathrm{Conv}_{w_m^{(l)}}^{\relu} & \text{(otherwise)}
    \end{cases}\\
    C_{m}&:=C_{m}^{(L_m)} \circ \cdots \circ C_{m}^{(1)}.
\end{align*}
Also, we denote by $C_{[m:m']}$ the subnetwork of $\mathrm{Conv}^{\relu}_{\bmtheta}$ between the $m$-th and $m'$-th block. That is,
\begin{align*}
    C_{[m:m']}:=\begin{cases}
        (C_{m'}+ I) \circ \cdots \circ (C_{m}+ I) & (\text{if $m \geq 1$})\\
        (C_{m'}+ I) \circ \cdots \circ (C_{1}+ I) \circ P & (\text{if $m =0$})
    \end{cases}
\end{align*}
for $m, m' = 0, \ldots, M$.
We define ${C'}_m^{(l)}$, $C'_m$ and $C'_{[m:m']}$ similarly for $\bmtheta'$.

\begin{prop}\label{prop:res-block-seq-sup-bound}
For $m\in [M]$ and $x\in [-1, 1]^D$, we have
$\|C_{[0:m]}(x)\|_\infty\leq (1\vee B^{\mathrm{(conv)}}) \varrho_m \varrho_m^{+}$.
Here, $\varrho_m = \left(\prod_{i=1}^{m} (1 + \rho_i)\right)$ and $\varrho_m^+ = \left(1 + \sum_{i=1}^m L_i\rho_{i}^+ \right)$ ($\rho_m$ and $\rho_m^+$ are constants defined in Theorem \ref{thm:estimation-extend}).
\end{prop}
\begin{proof}
By using Proposition \ref{prop:conv-sup-bound-multiple-layer} inductively, we have
\begin{align*}
    \|C_{[0:m]} (x)\|_\infty
    & \leq \|C_m (C_{[0:m-1]}(x)) + C_{[0:m-1]}(x)\|_\infty \\
    & \leq \|(1+\rho_m)C_{[0:m-1]}(x)+ B^{\mathrm{(conv)}}L_m\rho_{+m} ) \|_\infty\\
    & \leq (1+\rho_m)\|C_{[0:m-1]}(x)\|_\infty + B^{\mathrm{(conv)}}L_m\rho_{m}^+  \\
    & \cdots\\
    & \leq \|P (x)\|_\infty \prod_{i=1}^{m} (1 + \rho_i) + B^{\mathrm{(conv)}} \sum_{i=1}^{m} L_i\rho_i^{+} \prod_{j=i+1}^m (1+\rho_j) \\
    & \leq \prod_{i=1}^{m} (1 + \rho_i) + B^{\mathrm{(conv)}} \sum_{i=1}^{m} L_i\rho_i^{+} \prod_{j=i+1}^m (1+\rho_j)\\    
    & \leq (1\vee B^{\mathrm{(conv)}})\varrho_m \varrho_m^+.
\end{align*}
\end{proof}

\begin{lemma}\label{lem:sup-bound-of-epsilon-different-functions}
Let $\varepsilon > 0$. 
Suppose $\bmtheta$ and $\bmtheta'$ are within distance $\varepsilon$, that is,
$\max_{m, l} \|w_{m}^{(l)} - {w'}_{m}^{(l)}\|_\infty \leq \varepsilon$,
$\|b_{m}^{(l)} - {b'}_{m}^{(l)}\|_\infty \leq \varepsilon$,
$\|W - W'\|_\infty \leq \varepsilon$, and 
$\|b-b'\|_\infty \leq \varepsilon$.
Then, $\|\mathrm{CNN}_{\bmtheta}^{\relu}-\mathrm{CNN}_{\bmtheta'}^{\relu}\|_\infty \leq \Lambda_1\varepsilon$
where $\Lambda_1$ is the constant defined in Theorem \ref{thm:estimation-extend}.
\end{lemma}
\begin{proof}
For any $x\in [-1, 1]^D$, we have
\begin{align}
\left|\mathrm{CNN}_{\bmtheta}^{\relu}(x) - \mathrm{CNN}_{\bmtheta'}^{\relu}(x) \right| \nonumber
&= \left|\mathrm{FC}^{\id}_{W, b}\circ C_{[0:M]}(x)
       - \mathrm{FC}^{\id}_{W', b'}\circ C'_{[0:M]}(x)\right| \nonumber \\
&= \left| \left(\mathrm{FC}^\id_{W, b} - \mathrm{FC}^\id_{W', b'}\right) \circ C_{[0:M]}(x)\right| \nonumber \\
    & + \sum_{m=1}^M \left| \mathrm{FC}^\id_{W', b'} 
       \circ C_{[m+1:M]} \circ \left( C_m - C'_m \right) \circ C'_{[0:m-1]}(x)\right| \label{eq:f-g}.
\end{align}

We will bound each term of (\ref{eq:f-g}).
By Proposition \ref{prop:FC-diff-bound} and Proposition \ref{prop:res-block-seq-sup-bound},
\begin{align}\label{eq:f-g-bound-first-term}
\left|\left(\mathrm{FC}^{\id}_{W, b} - \mathrm{FC}^{\id}_{W', b'}\right) \circ C_{[0:M]}(x) \right|
& \leq (\|W\|_0 + \|W'\|_0)\|W-W'\|_\infty
\|C_{[0:M]}(x)\|_\infty \nonumber 
    + \|b-b'\|_\infty\nonumber\\
&\leq 2C^{(L_0)}_0D\|C_{[0:M]}(x)\|_\infty \varepsilon + \varepsilon\nonumber\\
&\leq 2C^{(L_0)}_0D (1\vee B^{\mathrm{(conv)}})\varrho_M \varrho_M^+ \varepsilon + \varepsilon\nonumber\\
&\leq 3C^{(L_0)}_0D (1\vee B^{\mathrm{(conv)}})\varrho_M \varrho_M^+ \varepsilon.
\end{align}

On the other hand, for $m\in[M]$,
\begingroup
\allowdisplaybreaks
\begin{align}\label{eq:f-g-bound-second-term}
& \left|\mathrm{FC}^{\id}_{W', b'}
    \circ C'_{[m+1:M]} \circ (C_m - C'_m)
    \circ C_{[0:m-1]} (x) \right| \nonumber \\
& \leq \|W'\|_0 \|W'\|_\infty \nonumber 
    \|C'_{[m+1:M]} \circ (C_m - C'_m)\circ C_{[1:m-1]}(x)\|_\infty \nonumber
    \quad \text{(by Proposition \ref{prop:fc-lipschitz-bound})}\nonumber \\
& \leq C^{(L_0)}_0D B^{\mathrm{(fc)}} \nonumber
    \|C'_{[m+1:M]} \circ (C_m - C'_m)\circ C_{[0:m-1]}(x)\|_\infty \nonumber \\
& \leq C^{(L_0)}_0D B^{\mathrm{(fc)}}\left(\prod_{i=m+1}^M \rho_i\right) \nonumber 
    \left\|(C_m - C'_m)\circ C_{[0:m-1]}(x) \right\|_\infty \nonumber 
    \quad \text{(by Proposition \ref{prop:conv-op-bound} and \ref{prop:conv-lipchitz-bound})}\nonumber \\
& \leq C^{(L_0)}_0D B^{\mathrm{(fc)}}\left(\prod_{i=m+1}^M \rho_i\right) \left(\rho_m \|C_{[0:m-1]}\|_\infty \varepsilon+ \varepsilon \right) \nonumber 
    \quad \text{(by Proposition \ref{prop:conv-op-bound} and \ref{prop:conv-diff-bound})}\nonumber \\
& \leq C^{(L_0)}_0D B^{\mathrm{(fc)}}\left(\prod_{i=m+1}^M \rho_i\right) \nonumber 
    \left(\rho_m (1\vee B^{\mathrm{(conv)}}) \varrho_{m-1} \varrho_{m-1}^+ + 1\right) \varepsilon \nonumber 
    \quad \text{(by Proposition \ref{prop:conv-sup-bound-multiple-layer})}\nonumber \\
& \leq 2C^{(L_0)}_0D B^{\mathrm{(fc)}}(1\vee B^{\mathrm{(conv)}}) \varrho_M \varrho_M^+ \varepsilon
\end{align}
\endgroup

By applying (\ref{eq:f-g-bound-first-term}) and (\ref{eq:f-g-bound-second-term}) to (\ref{eq:f-g}), we have
\begingroup
\allowdisplaybreaks
\begin{align*}
    |\mathrm{CNN}_{\bmtheta}^{\relu}(x) - \mathrm{CNN}_{\bmtheta'}^{\relu}(x)| 
    & \leq 3C^{(L_0)}_0D (1\vee B^{\mathrm{(conv)}})\varrho_M \varrho_M^+ \varepsilon \\
        & \quad + 2M C^{(L_0)}_0D B^{\mathrm{(fc)}}(1\vee B^{\mathrm{(conv)}}) \varrho_M \varrho_M^+ \varepsilon \\
    & \leq (2M+3)C^{(L_0)}_0D (1\vee B^{\mathrm{(fc)}}) (1\vee B^\mathrm{(conv)}) \varrho_M \varrho_M^+ \varepsilon\\
    & = \Lambda_1\varepsilon.
\end{align*}
\endgroup
\end{proof}

\subsubsection{Bounds for covering number of CNNs}\label{sec:covering-number-bound}

For a metric space $(\mathcal{M}_0, d)$ and $\varepsilon > 0$, we denote the (external) covering number of $\mathcal{M}\subset \mathcal{M}_0$ by $\mathcal{N}(\varepsilon, \mathcal{M}, d)$: $\mathcal{N}(\varepsilon, \mathcal{M}, d):= \inf \{N\in \bbN \mid \exists f_1, \ldots, f_N \in \mathcal{M}_0\ \mathrm{s.t.}\ \forall f\in \mathcal{M}, \exists n\in [N]\ \mathrm{s.t.}\ d(f, f_n) \leq \varepsilon \}$.

\begin{lemma}\label{lem:covering-number}
Let $B:=B^{\mathrm{(conv)}}\vee B^{\mathrm{(fc)}}$. For $\varepsilon >0$, we have $\mathcal{N}(\varepsilon, \mathcal{F}^{\mathrm{(CNN)}}, \|\cdot\|_\infty) \leq \left(2B\Lambda_1\varepsilon^{-1}\right)^{\Lambda_2}$.
\end{lemma}
\begin{proof}
The idea of the proof is same as that of Lemma 5 of~\citet{schmidt2017nonparametric}.
We divide the interval of each parameter range ($[-B^{\mathrm{(conv)}}, B^{\mathrm{(conv)}}]$ or $[-B^{\mathrm{(fc)}}, B^{\mathrm{(fc)}}]$) into bins with width $\Lambda_1^{-1}\varepsilon$ (i.e.,  $2B^{\mathrm{(conv)}}\Lambda_1\varepsilon^{-1}$ or $2B^{\mathrm{(fc)}}\Lambda_1\varepsilon^{-1}$ bins for each interval).
If $f, f'\in \mathcal{F}^{\mathrm{(CNN)}}$ can be realized by parameters such that every pair of corresponding parameters are in the same bin, then, $\|f-f'\|_\infty \leq \varepsilon$ by Lemma \ref{lem:sup-bound-of-epsilon-different-functions}.
We make a subset $\mathcal{F}_0$ of $\mathcal{F}^{\mathrm{(CNN)}}$ by picking up every combination of bins for $\Lambda_2$ parameters.
Then, for each $f\in \mathcal{F}^{\mathrm{(CNN)}}$, there exists $f_0\in \mathcal{F}_0$ such that $\|f-f_0\|_\infty \leq \varepsilon$.
There are at most $2B\Lambda_1 \varepsilon^{-1}$ choices of bins for each parameter.
Therefore, the cardinality of $\mathcal{F}_0$ is at most $\left(2B\Lambda_1\varepsilon^{-1}\right)^{\Lambda_2}$.
\end{proof}

\subsection{Proofs of Theorem \ref{thm:estimation} and Corollary \ref{cor:estimation-error-wrt-sample-size}}

We use the lemma in~\citet{schmidt2017nonparametric} to bound the estimation error of the clipped ERM estimator $\hat{f}$.
Since our problem setting is slightly different from the one in the paper, we restate the statement.

\begin{lemma}[cf.~\citet{schmidt2017nonparametric} Lemma 4]\label{lem:schmidt-hieber-lem10}
Let $\mathcal{F}$ be a family of measurable functions from $[-1, 1]^D$ to $\bbR$.
Let $\hat{f}$ be the clipped ERM estimator of the regression problem described in Section \ref{sec:problem-setting}.
Suppose the covering number of $\mathcal{F}$ satisfies $\mathcal{N}(N^{-1}, \mathcal{F}, \|\cdot\|_\infty)\geq 3$.
Then,
\begin{align*}
    \bbE_{\mathcal{D}}\|f^\circ - \hat{f}\|^2_{\mathcal{L}^2(\mathcal{P}_X)} \leq
    C \left(\inf_{f\in \mathcal{F}} \|f-f^\circ\|^2_{\mathcal{L}^2(\mathcal{P}_X)}
    + \log \mathcal{N}\left(\frac{1}{N}, \mathcal{F}, \|\cdot\|_\infty\right) \frac{\tilde{F}^2}{N}\right),
\end{align*}
where $C>0$ is a universal constant, $\tilde{F}:= \frac{R_{\mathcal{F}}}{\sigma} \vee \frac{\|f^{\circ}\|_\infty}{\sigma} \vee \frac{1}{2}$ and $R_\mathcal{F} := \sup \{\|f\|_\infty \mid f \in \mathcal{F} \}$.
\end{lemma}
\begin{proof}
Basically, we convert our problem setting to fit the assumptions of Lemma 4 of~\citet{schmidt2017nonparametric} and apply the lemma to it.
For $f:[-1, 1]^D \to [-\sigma \tilde{F}, \sigma \tilde{F}]$, we define $A[f]:[0, 1]^D\to [0, 2\tilde{F}]$ by $A[f](x'):=\frac{1}{\sigma}f(2x'-1)+\tilde{F}$.
Let $\hat{f}_1$ be the (non-clipped) ERM estimator of $\mathcal{F}$.
We define
$X':=\frac{1}{2}(X + 1)$,
$f'^{\circ}:= A[f^\circ]$,
$Y':= f'^{\circ}(X) + \xi'$,
$\mathcal{F}':= \{A[f] \mid f\in \mathcal{F}\}$,
$\hat{f}'_1:=A[\hat{f}_1]$, and
$\mathcal{D}':= ((x'_n, y'_n))_{n\in[N]}$ where $x'_n:=\frac{1}{2}(x_n + 1)$ and $y'_n:=f'^{\circ}(x'_n) + \frac{1}{\sigma}(y_n - f^\circ(x_n))$.
Then, the probability that $\mathcal{D}'$ is drawn from $\mathcal{P}'^{\otimes N}$ is same as 
the probability that $\mathcal{D}$ is drawn from $\mathcal{P}^{\otimes N}$ where $\mathcal{P}'$ is the joint distribution of $(X', Y')$.
Also, we can show that $\hat{f}'$ is the ERM estimator of the regression problem $Y'=f'^{\circ} + \xi'$ using the dataset $\mathcal{D'}$: $\hat{f}'_1\in \argmin_{f'\in \mathcal{F}'} \hat{\mathcal{R}}_{\mathcal{D}'}(f')$.
We apply the Lemma 4 of~\citet{schmidt2017nonparametric} with
$n\leftarrow N$,
$d\leftarrow D$,
$\varepsilon \leftarrow 1$,
$\delta \leftarrow \frac{1}{N}$,
$\Delta_n \leftarrow 0$,
$\mathcal{F}'\leftarrow \mathcal{F}$,
$F\leftarrow 2\tilde{F}$,
$\hat{f}\leftarrow \hat{f}'_1$ and use the fact that the estimation error of the clipped ERM estimator is no worse than that of the ERM estimator, that is, $\|f^\circ - \hat{f}\|^2_{\mathcal{L}^2(\mathcal{P}_X)} \leq \|f^\circ - \hat{f}_1\|^2_{\mathcal{L}^2(\mathcal{P}_X)}$ to conclude.
\end{proof}

\begin{proof}[Proof of Theorem \ref{thm:estimation-extend}]

By Lemma \ref{lem:covering-number}, we have $\log \mathcal{N} := \log \mathcal{N}(N^{-1}, \mathcal{F}^{\mathrm{(CNN)}}, \|\cdot\|_\infty)\leq \Lambda_2\log (2B\Lambda_1N)$, where $B=B^{\mathrm{(conv)}}\vee B^{\mathrm{(fc)}}$.
Therefore, by Lemma \ref{lem:schmidt-hieber-lem10},
\begin{align*}
    \|f^\circ - \hat{f}\|^2_{\mathcal{L}^2(\mathcal{P}_X)}
    &\leq C_0 \left(\inf_{f\in \mathcal{F}} \|f-f^\circ\|^2_{\mathcal{L}^2(\mathcal{P}_X)} + \log \mathcal{N} \frac{\tilde{F}^2}{N}\right)\\
    &\leq C_1 \left(\inf_{f\in \mathcal{F}^{\mathrm{(FNN)}}} \|f-f^\circ\|^2_\infty + \frac{\tilde{F}^2}{N}\Lambda_2\log (2B\Lambda_1 N) \right),
\end{align*}
where $C_0, C_1 > 0$ are universal constants.
We used in the last inequality the fact $\|\mathrm{clip}[f] - f^\circ\|_{\mathcal{L}^2(\mathcal{P}_X)} \leq \|\mathrm{clip}[f] - f^\circ\|_\infty  \leq \|f - f^\circ\|_\infty$ any $f\in \mathcal{F}^{\mathrm{(CNN)}}$ and the assumption $\mathcal{F}^{\mathrm{(FNN)}} \subset \mathcal{F}^{\mathrm{(CNN)}} $.
\end{proof}
As discussed at the beginning of this section, Theorem \ref{thm:estimation} is the special case of Theorem \ref{thm:estimation-extend}.

\begin{proof}[Proof of Corollary \ref{cor:estimation-error-wrt-sample-size}]
We only care about the order with respect to $N$ in the $O$-notation.
Set $M = \left\lfloor N^\alpha \right\rfloor$ for $\alpha> 0$.
Using the assumptions of the corollary, the estimation error is
\begin{align*}
    \|f^\circ - \hat{f}\|_{\mathcal{L}^2(\mathcal{P}_x)}^2 =
    \tilde{O}\left(\max \left( N^{-2\alpha\gamma_1}, N^{\alpha\gamma_2-1} \right)\right)
\end{align*}
by Theorem \ref{thm:estimation}.
The order of the right-hand side with respect to $N$ is minimized when $\alpha = \frac{1}{2\gamma_1 + \gamma_2}$.
By substituting $\alpha$, we can derive Corollary \ref{cor:estimation-error-wrt-sample-size}.
\end{proof}

\section{Proofs of Corollary \ref{cor:approximation-barron} and Corollary \ref{cor:estimation-barron}}

By Theorem 2 of~\cite{klusowski2018approximation}, for each $M\in \bbN_{+}$, there exists
\begin{align*}
    f^{\mathrm{(FNN)}}
    & := \frac{1}{M}\sum_{m=1}^{M} b_m (a_m^{\top} x - t_m)_+
    = \sum_{m=1}^{M} b_m \left(\frac{a_m^{\top}}{M}  x - \frac{t_m}{M}\right)_+
\end{align*}
with $|b_m| \leq 1$, $\|a_m\|_1 = 1$, and $|t_m| \leq 1$ such that $\|f^\circ- f^{\mathrm{(FNN)}}\|_\infty \leq C v_{f^\circ}\sqrt{\log M + D} {M}^{-\frac{1}{2} - \frac{1}{D}}$ where $v_{f^\circ} := \int_{\bbR^D} \|w\|_2^s \left| \mathcal{F}[f^\circ](w) \right| \mathrm{d}w$ and $C>0$ is a universal constant.
We set
\begin{align*}
    L_m\leftarrow 1, \quad
    D_m^{(1)}\leftarrow 1, \quad
    B^{\mathrm{(bs)}}\leftarrow M^{-1}, \quad
    B^{\mathrm{(fin)}}\leftarrow 1
    \end{align*}
($m\in [M]$) in the Theorem \ref{thm:fnn-to-cnn-extend}, then, we have $f^{\mathrm{(FNN)}} \in \mathcal{F}^{\mathrm{(FNN)}}_{\bmD_1, B^{\mathrm{(bs)}}, B^{\mathrm{(fin)}}}$.
By applying Theorem \ref{thm:fnn-to-cnn-extend}, there exists a CNN $f^{\mathrm{(CNN)}} \in \mathcal{F}^{\mathrm{(CNN)}}_{\bmC, \bmK, B^{\mathrm{(conv)}}, B^{\mathrm{(fc)}}}$ such that $f^{\mathrm{(FNN)}} = f^{\mathrm{(CNN)}}$.
Here,
$\bmC = (C_m^{(1)})_{m}$ with $C_m^{(1)}=4$,
$\bmK = (K_m^{(1)})_{m}$ with $K_m^{(1)}=K$,
$B^{\mathrm{(conv)}} = M^{-1} \vee M^{-\frac{1}{L_0}} = M^{-\frac{1}{L_0}}$, and 
$B^{\mathrm{(fc)}} = M$.
This proves Corollary \ref{cor:approximation-barron}.

With these evaluations, we have $\Lambda_1=O({M}^3)$ because $B^{\mathrm{(conv)}}=M^{-\frac{1}{L_0}}$ and hence
\begin{equation*}
    \prod_{m=0}^{M} (1+\rho_m) \lesssim (1 + M^{-\frac{1}{L_0}})^M \simeq e^{L_0} = O(1).
\end{equation*}
In addition, $B^{\mathrm{(conv)}}$ is $O(1)$ and $B^{\mathrm{(fc)}}$ is $O(M)$.
Therefore, we have $\log \Lambda_1 B = \tilde{O}(1)$.
Also, we have $\Lambda_2 = O(M)$. Therefore, we can apply Corollary \ref{cor:estimation-error-wrt-sample-size} with $\gamma_1 = \frac{1}{2} + \frac{1}{D}$, $\gamma_2 = 1$ to conclude.
\qed

\section{Proofs of Corollary \ref{cor:approximation-holder} and Corollary \ref{cor:estimation-holder}}

We first prove the scaling property of the FNN class.

\begin{lemma} \label{lem:scaling}
Let $M\in \bbN_{+}$, $L_m\in \bbN_{+}$, and $D_m^{(l)}\in \bbN_{+}$ for $m\in [M]$ and $l\in [L_m]$.
Let $B^{\mathrm{(bs)}}, B^{\mathrm{(fin)}} > 0$.
Then, for any $k\geq 1$, we have $\mathcal{F}^{\mathrm{(FNN)}}_{\bm{D}, B^{\mathrm{(bs)}}, B^{\mathrm{(fin)}}} \subset \mathcal{F}^{\mathrm{(FNN)}}_{\bm{D}, k^{-1}B^{\mathrm{(bs)}}, k^{L}B^{\mathrm{(fin)}}}$ where $L:=\max_{m\in [M]} L_m$ is the maximum depth of the blocks.
\end{lemma}
\begin{proof}
Let $\bmtheta = ((W_m^{(l)})_{m, l}, (b_m^{(l)})_{m, l}, (w_m)_m, b)$ be the parameter of an FNN and suppose that $\mathrm{FNN}^{\relu}_{\bmtheta} \in \mathcal{F}^{\mathrm{(FNN)}}_{\bm{D}, B^{\mathrm{(bs)}}, B^{\mathrm{(fin)}}}$.
We define $\bmtheta':=(({W'}_m^{(l)})_{m, l}, ({b'}_m^{(l)})_{m, l}, (w'_m)_m, b')$ by
\begin{align*}
    {W'}_m^{(l)} := k^{-\frac{L}{L_m}} W_m^{(l)}, \quad
    {b'}_m^{(l)} := k^{-l\frac{L}{L_m}} b_m^{(l)}, \quad
    w'_m := k^{L} w_m, \quad
    b' := b. 
\end{align*}
Since $k\geq 1$, we have $\mathrm{FNN}^{\relu}_{\bmtheta'} \in \mathcal{F}^{\mathrm{(FNN)}}_{\bm{D}, k^{-1}B^{\mathrm{(bs)}}, k^{L}B^{\mathrm{(fin)}}}$.
Also, by the homogeneous property of the ReLU function (i.e.,  $\relu(ax) = a\relu(x)$ for $a>0$), we have $\mathrm{FNN}^{\relu}_{\bmtheta} = \mathrm{FNN}^{\relu}_{\bmtheta'}$.
\end{proof}

Next, we prove the existence of a block-sparse FNN with constant-width blocks that optimally approximates a given $\beta$-H\"older function.
It is almost the same as the proof in~\citet{schmidt2017nonparametric}.
However, we need to construct the FNN to have a block-sparse structure.
\begin{lemma}[cf. \citet{schmidt2017nonparametric} Theorem 5 ]\label{lem:schmidt-hieber-thm5}
Let $\beta>0$, $M\in \bbN_{+}$ and $f^\circ: [-1, 1]^D\to \bbR$ be a $\beta$-H\"older function.
Then, there exists $D' = O(1)$, $L'=O(\log M)$, and a block-sparse FNN $f^{\mathrm{(FNN)}} \in \mathcal{F}^{\mathrm{(FNN)}}_{\bmD, 1, 2M\|f^\circ\|_\beta}$ such that $\|f^\circ - f^{\mathrm{(FNN)}}\|_\infty = \tilde{O}(M^{-\frac{\beta}{D}})$.
Here, we set $L_m := L'$ and $D_m^{(l)} := D'$ for all $m\in [M]$ and $l \in [L_m]$ and define $\bmD := (D_m^{(l)})_{m, l}$.
\end{lemma}
\begin{proof}
First, we prove the lemma when the domain of $f^\circ$ is $[0, 1]^D$.
Let $M'$ be the largest interger satisfying $(M'+1)^D \leq M$.
Let $\Gamma(M') = \left(\frac{\bbZ}{M'}\right)^D \cap [0, 1]^D = \{\frac{m'}{M'} \mid m' \in \{0, \ldots, M'\}^D \}$ be the set of lattice points in $[0, 1]^D$\footnote{\citet{schmidt2017nonparametric} used $\bm{D}(M')$ to denote this set of lattice points. We used different characters to avoid notational conflict.}.
Note that the cardinality of $\Gamma(M')$ is $(M'+1)^D$.
Let $P^{\beta}_a f^\circ$ be the Taylor expansion of $f^\circ$ up to order $\floor{\beta}$ at $a\in [0, 1]^D$:
\begin{align*}
    (P^\beta_a f^\circ)(x) = \sum_{0\leq |\alpha| < \beta} \frac{(\partial^{\alpha}f^\circ)(a)}{\alpha !}(x-a)^\alpha.
\end{align*}
For $a\in [0, 1]^{D}$, we define a hat-shaped function $H_a: [0, 1]^D\to [0, 1]$ by
\begin{align*}
    H_{a}(x) := \prod_{j=1}^D ({M'}^{-1} - |x_j - a_j|_+).
\end{align*}
Note that we have $\sum_{a \in \Gamma(M')} H_{a} (x) = 1$, i.e., they are a partition of unity.
Let $P^{\beta}f^\circ$ be the weighted sum of the Taylor expansions at lattice points of $\Gamma(M')$:
\begin{align*}
    (P^{\beta}f^\circ)(x) := {M'}^D \sum_{a\in D(M')} (P^{\beta}_{a}f^\circ)(x) H_{a}(x).
\end{align*}
By Lemma B.1 of \citet{schmidt2017nonparametric}, we have
\begin{align*}
    \|P^{\beta}f^\circ - f^\circ\|_{\infty} \leq \|f^\circ\|_\beta {M'}^{-\beta}.
\end{align*}
Let $m$ be an interger specified later and set $L^\ast:= (m+5)\ceil{\log_2 D}$.
By the proof of Lemma B.2 of \citet{schmidt2017nonparametric}, for any $a\in \Gamma(M')$, there exists an FNN $\mathrm{Hat}_{a}: [0, 1]^{D}\to [0, 1]$ whose depth and width are at most $2 + L^\ast$ and $6D$, respectively and whose parameters have sup-norm $1$, such that
\begin{align*}
    \left\|\mathrm{Hat}_{a} - H_{a} \right\|_\infty \leq 3^D 2^{-m}.
\end{align*}
Next, let $B := 2\|f^\circ\|_\beta$ and $C_{D, \beta}$ be the number of distinct $D$-variate monomials of degree up to $\floor{\beta}$.
By the equation (7.11) of \citet{schmidt2017nonparametric}, for any $a\in \Gamma(M)$, there exists an FNN $Q_{a}: [0, 1]^{D}\to [0, 1]$ \footnote{We prepare $Q_a$ for each $a\in \Gamma(M)$ as opposed to the original proof of \cite{schmidt2017nonparametric}, in which $Q_a$'s shared the layers the except the final one and were collectively denoted by $Q_1$.} whose depth and width are $1 + L^\ast$ and $6DC_{D, \beta}$ respectively and whose parameters have sup-norm $1$, such that
\begin{align*}
    \left\|Q_{a} - \left(\frac{P^\beta_{a}f^\circ}{B} + \frac{1}{2}\right)\right\|_\infty \leq 3^D 2^{-m}.
\end{align*}
Thirdly, by Lemma A.2 of \cite{schmidt2017nonparametric}, there exists an FNN $\mathrm{Mult}: [0, 1]^2\to [0, 1]$, whose depth and width are $m+4$ and $6$, respectively and whose parameters have sup-norm $1$ such that
\begin{align*}
    \left|\mathrm{Mult}(x, y)-xy\right| \leq 2^{-m}
\end{align*}
for any $x, y\in [0, 1]$.
For each $a\in \Gamma(M')$, we combine $\mathrm{Hat}_{a}$ and $Q_{a}$ using $\mathrm{Mult}$ and constitute a block of the block-sparse FNN corresponding to $a\in \Gamma(M)$ by $\mathrm{FC}_{a}:= \mathrm{Mult}(Q_{a}(\cdot), \mathrm{Hat}_{a}(\cdot))$.
Then, we have
\begin{align*}
    \left\|\mathrm{FC}_{a} - \left(\frac{P^\beta_{a}f^\circ}{B} + \frac{1}{2}\right) H_{a} \right\|_\infty
    \leq 2^{-m} + 3^D2^{-m} + 3^D2^{-m}
    \leq 3^{D+1}2^{-m}.
\end{align*}
We define $f^{\mathrm{(FNN)}}(x) := \sum_{a \in \Gamma(M)} (B{M'}^{D} \mathrm{FC}_{a}(x)) - \frac{B}{2}$.
By construction, $f^{\mathrm{(FNN)}}$ is a block-sparse FNN with $(M'+1)^D (\leq M)$ blocks each of which has depth and width at most $L':=2+L^\ast + (m+4)$ and $D':=6(C_{D, \beta}+ 1)D$, respectively.
The norms of the block-sparse part and the finally fully-connected layer are $1$ and $B{M'}^D (\leq BM)$, respectively.
In addition, we have
\begin{align*}
    & |f^{\mathrm{(FNN)}}(x) - (P^{\beta}f^\circ)(x)|\\
    & \leq \sum_{a \in \Gamma(M)}B{M'}^{D} \left|\mathrm{FC}_{a}(x) -   \left(\frac{(P^\beta_{a}f^\circ)(x)}{B}+\frac{1}{2}\right) H_{a}(x)\right|
     + \frac{B}{2}\left|1 - {M'}^D \sum_{a \in \Gamma(M')}H_{a} (x) \right|\\
    &\leq ({M'}+1)^D \times B{M'}^D 3^{D+1}2^{-m}\\
    &\leq 3^{D+1}2^{-m}BM^2
\end{align*}
for any $x\in [0, 1]^D$.
Therefore,
\begin{align*}
    |f^{\mathrm{(FNN)}}(x) - f^\circ(x)|
    & \leq |f^{\mathrm{(FNN)}} - (P^{\beta}f^\circ)(x)| + |(P^{\beta}f^\circ)(x) - f^\circ(x)|\\
    & \leq 3^{D+1}2^{-m}BM^2 + \|f^\circ\|_\beta {M'}^{-\beta}\\
    & \leq 2 \cdot 3^{D+1} 2^{-m} \|f^\circ\|_\beta M^2 + \|f^\circ\|_\beta M^{-\frac{\beta}{D}}.
\end{align*}
We set $m=\ceil{\log_2 M^{2 + \frac{\beta}{D}}}$, then, we have $L'=O(\log M)$, $D'=O(1)$, and
\begin{align*}
    \|f^{\mathrm{(FNN)}} - f^{\circ}\|\leq \|f^\circ\|_\beta(2\cdot3^{D+1} + 2^\beta)M^{-\frac{\beta}{D}}.
\end{align*}
By the defnition of $f^{\mathrm{(FNN)}}$ we have $f^{\mathrm{(FNN)}} \in \mathcal{F}^{\mathrm{(FNN)}}_{\bmD, 1, 2\|f^\circ\|_\beta M}$.

When the domain of $f^{\circ}$ is $[-1, 1]^D$, we should add the function $x \mapsto \frac{1}{2}(x+1) = \frac{1}{2}(x+1)_+ - \frac{1}{2}(-x-1)_+$ as a first layer of each block to fit the range into $[0, 1]^D$.
Specifically, suppose the first layer of $m$-th block in $f^{\mathrm{(FNN)}}$ is $x \mapsto \relu(Wx - b)$, then the first two layers become $x\mapsto \relu(\begin{bmatrix}\frac{1}{2} (x + 1) &-\frac{1}{2}(x+1)\end{bmatrix})$ and $\begin{bmatrix}y_1 & y_2\end{bmatrix} \mapsto \relu(Wy_1 - Wy_2 - b)$, respectively.
Since this transformation does not change the maximum sup norm of parameters in the block-sparse and the order of $L'$ and $D'$, the resulting FNN still belongs to $\mathcal{F}^{\mathrm{(FNN)}}_{\bmD, 1, 2\|f^\circ\|M}$.
\end{proof}

\begin{proof}[Proofs of Corollary~\ref{cor:approximation-holder} and Corollary~\ref{cor:estimation-holder}]
In this proof, we only care about the dependence on $M$ in the $O$-notation.
Let $\tilde{M}:=2\|f^\circ\|_\beta M$.
By Lemma \ref{lem:schmidt-hieber-thm5}, there exists $f^{\mathrm{(FNN)}} \in \mathcal{F}^{\mathrm{(FNN)}}_{\bmD, 1, \tilde{M}}$ such that $\|f^{\mathrm{(FNN)}} - f^{\circ}\|_\infty = O(M^{-\frac{\beta}{D}})$ ($L'$, $D'$, and $\bmD$ as in Lemma \ref{lem:schmidt-hieber-thm5}).
Let
\begin{equation*}
    k := \left(\frac{16D'K}{M^{\frac{1}{L'}} \wedge 1}\right)^{L_0}
       = \left(\frac{16D'K}{e^{\frac{1}{C'}}\wedge 1}\right)^{L_0},
\end{equation*}
where $C'$ is a constant such that $L'=C'\log M$.
We note $k\geq 1$.
Using Lemma \ref{lem:scaling}, there exists $\tilde{f}^{\mathrm{(FNN)}} \in \mathcal{F}^{\mathrm{(FNN)}}_{\bmD, k^{-1}, k^{L'}\tilde{M}}$
such that $\tilde{f}^{\mathrm{(FNN)}} = f^{\mathrm{(FNN)}}$.
We apply Theorem \ref{thm:fnn-to-cnn-extend} to $\mathcal{F}^{\mathrm{(FNN)}}_{\bmD, k^{-1}, k^{L'}\tilde{M}}$ and find $f^{\mathrm{(CNN)}}\in \mathcal{F}^{\mathrm{(CNN)}}_{\bmC, \bmK, B^{\mathrm{(conv)}}, B^{\mathrm{(fc)}}}$ where $\bmC := (C_m^{(l)})_{m\in[M], l\in [L_m]}$ and $\bmK := (K_m^{(l)})_{m\in[M], l\in [L_m]}$ such that
\begin{align*}
    &L \leq M(L' + L_0),\\
    & C_{m}^{(l)} \leq 4D',\\
    & K_{m}^{(l)} \leq K,\\
    &B^{\mathrm{(conv)}} = k^{-1} \vee k^{-\frac{1}{L_0}} = k^{-\frac{1}{L_0}},\\
    &B^{\mathrm{(fc)}} = k^{L'}\tilde{M} (1 \vee k^{\frac{1}{L_0}}) = k^{L'+\frac{1}{L_0}}\tilde{M},
\end{align*}
and $f^{\mathrm{(CNN)}} = \tilde{f}^{\mathrm{(FNN)}}$.
This proves Corollary~\ref{cor:approximation-holder}.

To prove Corollary~\ref{cor:estimation-holder}, we evaluate $\log \Lambda_1 (B^{\mathrm{(conv)}}\vee B^{\mathrm{(fc)}})$ and $\Lambda_2 = O(M\log M)$.
By the definition of $k$ and the bound on $C_{m}^{(l)}$ and $K_{m}^{(l)}$, we have $C_m^{(l-1)} K_m^{(l)} k^{-\frac{1}{L_0}} \leq \frac{1}{4}M^{-\frac{1}{L'}}$.
Therefore, we have
\begin{equation*}
    \rho_m \leq \prod_{l=1}^{L'} C_m^{(l-1)} K_m^{(l)} k^{-1} \leq M^{-1}
\end{equation*}
and hence $\prod_{m=0}^{M} (1 + \rho_m) = O(1)$.
Since $C_m^{(l-1)} K_m^{(l)} k^{-1} \leq \frac{1}{2}$ for sufficiently large $M$, we have $\rho_{m}^{+} = 1$ for sufficiently large $M$.
By definition, we have $B^{\mathrm{(conv)}} = O(1)$ and 
\begin{equation*}
    \log B^{\mathrm{(fc)}} = \left(L' + \frac{1}{L_0}\right) k + \log(\tilde{M}) = O(\log M).
\end{equation*}
Therefore, we have $\log (B^{\mathrm{(conv)}}\vee B^{\mathrm{(fc)}}) = \tilde{O}(1)$.
Combining these evaluations, we have $\log \Lambda_1 (B^{\mathrm{(conv)}}\vee B^{\mathrm{(fc)}}) = \tilde{O}(1)$.
For $\Lambda_2$, we can bound it by $\Lambda_2 = O(M\log M)$ using bounds for $C^{(l)}_m$, $K^{(l)}_m$ and $L'$.
Therefore, we can apply Corollary~\ref{cor:estimation-error-wrt-sample-size} with $\gamma_1 = \frac{\beta}{D}$, $\gamma_2 = 1$ and obtain the desired estimation error.
Since we set $M = O(N^{\frac{1}{2\gamma_1 + \gamma_2}})$, as in the proof of Corollary~\ref{cor:estimation-error-wrt-sample-size}, we can derive the bounds for $L_m$ with respect to $N$.
\end{proof}

\section{Proofs of Theorem \ref{thm:constant-depth-resnet-approximation} and Theorem \ref{thm:constant-depth-resnet-estimation}}

\begin{lemma}\label{lem:resblock-division}
Let $L, L', C', K'\in \bbN_+$ and $B > 0$.
Suppose we can realize $f + \mathrm{id}: \bbR^{D\times C'}\to \bbR^{D\times C'}$ with a residual block with an identity connection
whose depth, channel size, and filter size are $L'$, $C'$, and $K'$, respectively and whose parameter norm is bounded by $B$.
Let $S_0 = \lceil \frac{L'}{L} \rceil$.
Then, there exist $S = 2S_0 - 1$ functions $\tilde{f}_1, \ldots, \tilde{f}_S: \bbR^{D\times 3C'}\to \bbR^{D\times 3C'}$ and $S$ masks $z_1, \ldots, z_S\in \{0, 1\}^{3C'}$,
such that
$f_s$ is realizable by a residual block whose depth, channel size, filter size, and parameter norm bound  are $L$, $3C'$, $K'$, and $B$, respectively and
$\tilde{f} := (\tilde{f}_S + J_S)  \circ \cdots \circ (\tilde{f}_1 + J_1): \bbR^{D\times 3C'}\to \bbR^{D\times 3C'}$
satisfies $\tilde{f}(\begin{bmatrix}x & 0 & 0\end{bmatrix}) = \begin{bmatrix}f(x) & 0 & 0\end{bmatrix}$.
Here $J_s$ is a channel-wise mask operation made from $z_s$.
\end{lemma}
\begin{proof}
We divide the residual block representing $f$ into $S_0$ CNNs with depth at most $L$ and denote them sequentially by $g_1, \ldots, g_{S_0}$ so that $f = g_{S_0}\circ \cdots \circ g_1$.
We define $\tilde{g}_s: \bbR^{D\times 3C'}\to \bbR^{D\times 3C'}$ ($s\in [S_0]$) from $g_s$ by
\begin{align*}
\tilde{g}_s([x_1\ x_2\ x_3])
    =\begin{cases}
        [0\ y_1\ 0]\ \text{(if $s = 1$)}\\
        [0\ y_3\ 0]\ \text{(if $s \not = 1, S_0$ and odd)}\\
        [0\ 0\ y_2]\ \text{(if $s \not = 1, S_0$ and even)}\\
        [y_3\ 0\ 0]\ \text{(if $s = S_0$ and odd)}\\
        [y_2\ 0\ 0]\ \text{(if $s = S_0$ and even)}
    \end{cases},
\end{align*}
where $y_i = g_s(x_i)$ ($i = 1, 2, 3$).
Note that we can construct $\tilde{g}_s$ by a residual block with depth $L$, channel size $3C'$, filter size $K'$, and parameter norm $B$.
Next, we define $u_s$ ($s\in [S_0 - 1]$) by
\begin{align*}
u_s =
    \begin{cases}
        \begin{bmatrix}1 & 1 & 0 \end{bmatrix}^{\top} \quad \text{(if $s$: odd)}\\
        \begin{bmatrix}1 & 0 & 1 \end{bmatrix}^{\top} \quad \text{(if $s$: even)}
    \end{cases}
\end{align*}
Then, we define $\tilde{f} := (\tilde{g}_{S_0} + \mathrm{id}) \circ (0 + J'_{S_0-1}) \circ (\tilde{g}_{S_0-1} + \mathrm{id}) \circ (0 + J'_{1}) \circ (\tilde{f}_1 + \mathrm{id})$
where $J'_s$ is a channel-wise mask constructed from $u_s$ and $0: \bbR^{D\times 3C'}\to \bbR^{D\times 3C'}$ is a constant zero function, which is obviously representable by a residual block.
By definition, $\tilde{f}$ is realizable by $S$ residual blocks with channel-wise masking identity connections and satisfies the conditions
on the depth, channel size, filter size, and norm bound.
\end{proof}

\begin{proof}[Proof of Theorem \ref{thm:constant-depth-resnet-approximation}]
The first part of the proof is the same as that of Corollary \ref{cor:approximation-holder},
except that we define $k$ using $L$ instead of $L'$ that is,
\begin{equation*}
    k=\left(\frac{16D'K}{M^{\frac{1}{L}} \wedge 1}\right)^{L_0}.
\end{equation*}
Here, $D'$ is a constant satisfying $D' = O(1)$ as a function of $M$.
Then, there exists a CNN $\tilde{f}^{\mathrm{(CNN)}} \in \mathcal{F}^{\mathrm{(CNN)}}_{M, L', C', K', B^{\mathrm{(conv)}}, B^{\mathrm{(fin)}}}$ such that
$\|\tilde{f}^{\mathrm{(CNN)}} - f^{\circ} \| = O({M}^{-\frac{\beta}{D}})$.
The parameter of the set of CNNs satisfy
$L' = O(\log M)$ $C' \leq 4D'$, $K' \leq K$, $B^{\mathrm{(conv)}} = k^{-\frac{1}{L_0}}$, and $B^{\mathrm{(fc)}} = 2\|f^{\circ}\|_\beta k^{L'}M$.
We apply Lemma \ref{lem:resblock-division} to each residual block of $\tilde{f}^{\mathrm{(CNN)}}$.
Then, there exists $f^{\mathrm{(CNN)}} \in \mathcal{G}_{\tilde{M}, L, C, K, B^{\mathrm{(conv)}}, B^{\mathrm{(fin)}}}$ such that $f^{\mathrm{(CNN)}} = \tilde{f}^{\mathrm{(CNN)}}$
and $\tilde{M} = M\lceil\frac{L'}{L}\rceil$, $C \leq 3C'$, $K'\leq K$, $B^{\mathrm{(conv)}} = k^{-\frac{1}{L_0}}$, and $B^{\mathrm{(fc)}} = 2\|f^{\circ}\|_\beta k^{L'+1}M$.
\end{proof}

Before going to the proof of Theorem \ref{thm:constant-depth-resnet-estimation}, we first note that
the definitions of $\Lambda_1$ and $\Lambda_2$ in Theorem \ref{thm:estimation} are valid even if we replace $\mathcal{F}^{\mathrm{(CNN)}}_{\tilde{M}, L, C, K, B^{\mathrm{(conv)}}, B^{\mathrm{(fin)}}}$ with $\mathcal{G} = \mathcal{G}_{\tilde{M}, L, C, K, B^{\mathrm{(conv)}}, B^{\mathrm{(fin)}}}$.

\begin{lemma}
Let $\tilde{M}, L, C, K\in \bbN_{+}$ and $B^{\mathrm{(conv)}}, B^{\mathrm{(fin)}}, \varepsilon > 0$.
Set $B = B^{\mathrm{(conv)}} \vee B^{\mathrm{(fin)}}$.
Then, the covering number of $\mathcal{G}$ with respect to the sup-norm
$\mathcal{N}(\varepsilon, \mathcal{G}, \|\cdot\|_\infty)$ is bounded by $(2B \Lambda_1\varepsilon^{-1})^{\Lambda_2} \cdot 2^{C\tilde{M}L}$,
where $\Lambda_1 = \Lambda_1(\mathcal{G})$ and $\Lambda_2=\Lambda_2(\mathcal{G})$ are ones defined in Theorem \ref{thm:estimation}, except that $\mathcal{F}^{\mathrm{(CNN)}}$ is replaced with $\mathcal{G}$.
\end{lemma}
\begin{proof}
First, we note that we can apply the same inequalities in Section \ref{sec:bounds-for-conv-layers} -- \ref{sec:bounds-for-residual-block} and Proposition \ref{prop:res-block-seq-sup-bound} to CNNs in $\mathcal{G}$.
Therefore, if two masked CNNs $f, g \in \mathcal{G}$ have the same masking patterns in identity connections and the distance of each pair of corresponding parameters in residual blocks is at most $\varepsilon$, then we can show $\|f - g\|_\infty \leq \Lambda_1\varepsilon$ in the same way as Lemma \ref{lem:sup-bound-of-epsilon-different-functions}.
Therefore, by the same argument of Lemma \ref{lem:covering-number}, the covering number of the subset of $\mathcal{G}$ consisting of CNNs with a specific masking pattern is bounded by $(2B\Lambda_1\varepsilon^{-1})^{\Lambda_2}$.
Since each CNN in $\mathcal{G}$ has $C\tilde{M}L$ parameters in identity connections which take values in $\{0, 1\}$, there are $2^{C\tilde{M}L}$ masking patterns.
Therefore, we have $\mathcal{N}(\varepsilon, \mathcal{G}, \|\cdot\|_\infty) \leq (2B \Lambda_1\varepsilon^{-1})^{\Lambda_2} \cdot 2^{C\tilde{M}L}$.
\end{proof}

The strategy for the proof of Theorem \ref{thm:constant-depth-resnet-estimation} is almost same as the proofs for Theorem \ref{thm:estimation-extend} and Corollary \ref{cor:estimation-holder},
except that we should replace $\Lambda_2\log (2B\Lambda_1 N)$ in (\ref{eq:estimation-bound-2}) with $\Lambda_2\log (2B\Lambda_1 N) + C\tilde{M}L \log 2$ (and $\Lambda_1$ and $\Lambda_2$ are defined via $\mathcal{G}$ instead of $\mathcal{F}^{\mathrm{(CNN)}}$).
However, the second term is at most in the same order (up to logarithmic factors) as the first one in our situation.
Therefore, we can derive the same estimation error rate.
\begin{proof}[Proof of Theorem \ref{thm:constant-depth-resnet-estimation}]
Take $\mathcal{G}$ as in the proof of Theorem \ref{thm:constant-depth-resnet-approximation}.
Let $\log \mathcal{N} := \log \mathcal{N}(N^{-1}, \mathcal{G}, \|\cdot\|_\infty)$.
By Lemma \ref{lem:schmidt-hieber-lem10}, we have
\begin{align*}
    \|f^\circ - \hat{f}\|^2_{\mathcal{L}^2(\mathcal{P}_X)} 
    & \leq C_0 \left(\inf_{f\in \mathcal{F^{\mathrm{(FNN)}}}} \|f-f^\circ\|^2_\infty \right.
    + \left. \frac{\tilde{F}^2}{N} \left(\Lambda_2\log (2B\Lambda_1 N) + C\tilde{M}L \log 2 \right) \right),
\end{align*}
where $C_0 > 0$ is a universal constant.
The first term in the outer-most parenthesis is $O({M}^{-\frac{\beta}{D}})$ by Lemma \ref{lem:schmidt-hieber-thm5}.
We will evaluate the order of the second term.
First, we have $\Lambda_2 = O(\tilde{M}) = \tilde{O}(M)$ by the definition of $\Lambda_2$.
By the definition of $k$, we have $\rho \leq {M}^{-1}$ and $\rho^+ = 1$ for sufficiently large $M$ therefore, $\varrho = O(1)$ and $\varrho^{+} = O(M)$ for sufficiently large $M$.
Again, by the definition of $k$, we have $B^{\mathrm{(conv)}}=O(1)$ and $B^{\mathrm{(fc)}} = O(M)$. 
Therefore, we have $\Lambda_1 = O({M}^3)$ and $B = O(M)$ and hence $\Lambda_2\log (2B \Lambda_1 N) = \tilde{O}(MN)$.
On the other hand, since $C=O(1), \tilde{M}=\tilde{O}(M), L=O(1)$, we have $C\tilde{M}L\log 2 = \tilde{O}(M)$.

Therefore, by setting $M = \left\lfloor N^\alpha \right\rfloor$ for $\alpha> 0$, the estimation error is
\begin{align*}
    \|f^\circ - \hat{f}\|_{\mathcal{L}^2(\mathcal{P}_x)}^2 =
    \tilde{O}\left(\max \left( N^{-2\alpha\gamma_1}, N^{\alpha\gamma_2-1} \right)\right),
\end{align*}
where $\gamma_1 = \frac{\beta}{D}$ and $\gamma_2 = 1$.
The order of the right-hand side with respect to $N$ is minimized when $\alpha = \frac{1}{2\gamma_1 + \gamma_2}$.
By substituting $\alpha$, we can derive the theorem.
\end{proof}

\section{One-sided padding vs. Equal-padding}\label{sec:padding-style}

In this paper, we adopted one-sided padding, which is not used so often practically, to simplify proofs.
However, with slight modifications, all statements are true for equally-padded convolutions, a widely employed padding style that adds (approximately) the same numbers of zeros to both ends of an input signal, with the exception that the filter size $K$ is restricted to $K \leq \left\lfloor\frac{D}{2} \right\rfloor$ instead of $K \leq D$.

\section{Difference between Original ResNet and Ours}\label{sec:resnet-comparison}

Aside from the number of layers, there are several differences between the CNN in this paper and the original ResNet~\cite{he2016deep}.
The most critical one is that our CNN does not have pooling nor Batch Normalization layers~\cite{ioffe15}.
We will consider a scaling scheme simpler than Batch Normalization to derive the optimality of CNNs with constant-depth residual blocks (see Definition \ref{def:masked-cnn-def}).
It is left for future research whether our result can extend to the ResNet-type CNNs with pooling or other scaling layers such as Batch Normalization.

\end{document}